\documentclass[11pt, oneside, a4paper]{article}  %
\usepackage{graphicx}
\usepackage{placeins}	%
\usepackage{amsthm}
\usepackage{amsfonts}
\usepackage{amscd}
\usepackage{amsbsy,amsxtra}
\usepackage{mathtools}
\usepackage{amssymb}
\usepackage{amsmath}
\usepackage{stmaryrd}   %
\usepackage{pifont}  %
\usepackage{dsfont}
\usepackage{bm}
\usepackage{accents}
\usepackage{enumerate}
\usepackage[utf8]{inputenc} 
\usepackage{physics}
\usepackage{caption}
\usepackage[raggedright]{titlesec}
\usepackage{xcolor}
\usepackage{array}
\newcolumntype{C}[1]{>{\centering\arraybackslash}m{#1}} %

\usepackage{multicol}
\usepackage{algorithm}
\usepackage{algpseudocode}

\usepackage{tikz}
\usetikzlibrary{patterns, patterns.meta, angles, quotes, calc, bbox, shapes.misc}
\usetikzlibrary{decorations.markings, decorations.pathmorphing}
\usetikzlibrary{intersections}
\usetikzlibrary{backgrounds}

\usepackage{pgfplots}
\usepgfplotslibrary{fillbetween}

\usetikzlibrary{external}
\usepgfplotslibrary{external}
\colorlet{Pcolor}{black!40!white}
\colorlet{PcolorLight}{black!20!white}
\tikzstyle{fill_P_side}=[fill=Pcolor, draw=none]
\tikzstyle{fill_P_side_light}=[fill=PcolorLight, draw=none]
\definecolor{greenDrawIo}{HTML}{82B366}
\definecolor{redDrawIo}{HTML}{B85450}
\definecolor{blueDrawIo}{HTML}{6C8EBF}
\definecolor{ecs100}{RGB}{86,170,28}

\usetikzlibrary{shadings}
\pgfdeclarehorizontalshading{shadePside}  %
  {100bp}{color(0bp)=(red); color(25bp)=(orange); color(50bp)=(yellow); color(75bp)=(green); color(100bp)=(blue)}

\usepackage[
	english=british
]{csquotes}
\usepackage[
    backend=biber,
    sortlocale=auto,
    natbib=true,	%
    url=false,
    maxbibnames=9,
    maxcitenames=2,
    uniquelist=false,	%
    style=authoryear 	%
]{biblatex}

\usepackage{thmtools}   %
\usepackage{hyperref}
\hypersetup{hypertexnames=false} %
\usepackage{cleveref}

\def\r#1{{\color{redDrawIo}#1}}

\declaretheorem[numberwithin=section,
	name=Theorem,
	]{theorem}
\declaretheorem[numberlike=theorem,
	name=Lemma,
	]{lemma}

\declaretheorem[numberlike=theorem,
	name=Corollary,
	]{corollary}
	
\declaretheorem[numberlike=theorem,
	name=Definition,
	style=definition,
	]{definition}
	
\declaretheorem[numberlike=theorem,
	name=Remark,
	style=remark,
	numbered=no,
	]{remark}

\newcommand{\NN}{\ensuremath{\mathds{N}}}  %
\newcommand{\RR}{\ensuremath{\mathds{R}}}  %
\newcommand{\cP}{\ensuremath{\mathcal{P}}} %
\newcommand{\ones}{\ensuremath{\mathds{1}}}
\newcommand{\I}[1]{\ones_{#1}}

\newcommand{\Eb}{E_b} %
\newcommand*\blank{{\mkern 2mu\cdot\mkern 2mu}}
\newcommand{\cv}[1]{\ensuremath{\begin{pmatrix}#1\end{pmatrix}}}    %
\newcommand{\cvb}[1]{\ensuremath{\begin{bmatrix}#1\end{bmatrix}}}   %

\DeclareMathOperator*{\aff}{aff}

\DeclareMathOperator*{\interior}{int}
\DeclareMathOperator*{\exterior}{out}

\DeclareMathOperator*{\Span}{span}

\DeclareMathOperator{\relu}{ReLU} %

\DeclareMathOperator{\CPA}{CPA}   %
\DeclareMathOperator{\CPL}{CPL}   %
\newcommand{\cl}[1]{\overline{#1}}  %

\DeclarePairedDelimiter{\set}{\lbrace}{\rbrace}
\DeclarePairedDelimiter\ceil{\lceil}{\rceil}
\DeclarePairedDelimiter\floor{\lfloor}{\rfloor}

\makeatletter
\newcommand{\mytag}[2]{%
  \text{#1}%
  \@bsphack
  \begingroup
    \@onelevel@sanitize\@currentlabelname
    \edef\@currentlabelname{%
      \expandafter\strip@period\@currentlabelname\relax.\relax\@@@%
    }%
    \protected@write\@auxout{}{%
      \string\newlabel{#2}{%
        {#1}%
        {\thepage}%
        {\@currentlabelname}%
        {\@currentHref}{}%
      }%
    }%
  \endgroup
  \@esphack
}
\makeatother

\newcommand{\titleText}{Linear-Size Neural Network Representation of Piecewise Affine Functions in $\RR^2$}

\title{\titleText}

\author{ \href{https://orcid.org/0009-0001-9695-3812}{\includegraphics[scale=0.06]{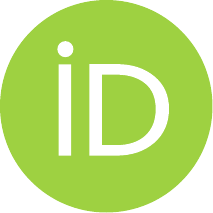}\hspace{1mm}Leo Zanotti}%
		\\
	Institute of Optimization and Operations Research\\
	Ulm University, Germany\\
	\texttt{leo.zanotti@uni-ulm.de}
}

\hypersetup{
	pdftitle={\titleText},
	pdfauthor={Leo Zanotti},
}

\begin{document}
	\maketitle
	\begin{abstract}
	It is shown that any continuous piecewise affine (CPA) function $\RR^2\to\RR$ with $p$ pieces can be represented by a ReLU neural network with two hidden layers and $O(p)$ neurons. 
	Unlike prior work, which focused on convex pieces, this analysis considers CPA functions with connected but potentially non-convex pieces.
\end{abstract}

	\section{Introduction}
One way to assess a neural network's ability to adapt to complex data is to examine the functions it can compute.
This paper considers this question for neural networks with the Rectified Linear Unit ($\relu$) activation function, focusing on networks that map from $\RR^2$.

As compositions of the piecewise linear $\relu$ and affine functions, the input-output mappings of these networks are described by continuous piecewise affine (CPA) functions. 
That is, for every neural network, there exists a finite cover of the input space such that the restriction of the function to any set in the cover is affine.
In this work, these sets are called pieces and must have connected interiors, though they may be non-convex and contain holes.
The main contribution of this work is a linear bound on the size of a neural network architecture required to represent a CPA function in $\RR^2$ with a given number of pieces. 

\begin{theorem}%
	\label{thm:NN_representation_intro}
	Any continuous piecewise affine function $f:\RR^2\to\RR$ with $p$ pieces can be represented by a $\relu$ neural network with two hidden layers and $O(p)$ neurons.
\end{theorem}
The constructed network is highly sparse, with the number of parameters also scaling as $O(p)$.
It seems unlikely that fewer than $O(p)$ parameters suffice to represent every CPA function with $p$ pieces.
This suggests that, for this task in two dimensions, networks with more than two hidden layers are not more powerful than those with two hidden layers.

To prove \autoref{thm:NN_representation_intro}, a representation of CPA functions in terms of maxima is introduced.
A univariate CPA function $f:\RR\to\RR$ defined over $p$ intervals can be expressed as
\begin{equation*}
	f(x)=\sum_{i=1}^{p-1}\sigma_i\max(a_ix+b_i, a_i'x+b'_i),
\end{equation*}
where $\sigma_i\in\set{-1,1}$ and $a_i,a_i',b_i,b_i'\in\RR$. 
\citet{mukherjee2017} observed that in two dimensions, maxima of only two affine functions are insufficient as summands.  
However, the following theorem shows that using three affine functions per summand suffices while keeping the number of summands linear in the number of pieces.
\begin{theorem}%
	Let $f:\RR^2\to\RR$ be a continuous piecewise affine function with $p$ pieces.
	Then, there exist affine functions $f^{(k)}_n$ and signs $\sigma_n^{(k)}\in\set{-1,1}$, such that
	\begin{equation*}
		f(x) = \sum_{n=1}^{9p} \sigma_{n}^{(1)}\max(f_{n}^{(1)},\sigma_{n}^{(2)}\max(f_{n}^{(2)},f_{n}^{(3)})).
	\end{equation*}
\end{theorem}
To achieve this representation, a constructive proof for a conic decomposition of the pieces is provided (see \citet{Shephard1967AngleSums} for conic decompositions of convex polytopes).

\subsection{Related Work}
Various works estimate the complexity that a given neural network architecture may achieve.
For a $\relu$ network, given that it computes a CPA function, the number of pieces of the represented function is a natural measure for its complexity.
Here, the number of pieces is defined as the minimum required, as the set of pieces is not unique.

The first significant lower and upper bounds on the complexity of neural networks were established by \citet{Montufar2014} and \citet{Montufar2017Notes}, respectively, and were later slightly improved by \citet{Serra2018BoundingAndCounting}. 
A unifying view on deriving such bounds was presented by \citet{Hinz2019}. 
Both upper and lower bounds indicate that asymptotically deep networks with a fixed number of neurons can produce exponentially more pieces than shallow ones. 
Moreover, \citet{arora2018understanding} show an exponential tradeoff between width and depth for a whole continuum of 'hard' functions.
These findings align with practitioners' observations that deeper networks are more powerful.

However, the lower bounds on the maximum achievable complexity of an architecture are based on very specific constructions.
Therefore, there might be much more functions of the same complexity that can not be represented by architectures of comparable size.

In contrast, this work focuses on architectures capable of representing \emph{all} CPA functions of given complexity.
\citet{arora2018understanding} established that this is feasible, which lead to research into the architectural size required for such representations.
Key contributions include results by \citet{brandenburg2024decompositionPolyhedra, Chen2022neurCompl, Hanin2019, He2020FEM, Hertrich2023TowardsLowerBounds, Koutschan2023}.

More specifically, \citet{arora2018understanding} demonstrated that any CPA function $f:\RR^d\to\RR^m$ can be represented using $\ceil{\log(d+1)}$ hidden layers.
\citet{Hertrich2023TowardsLowerBounds} conjectured that this bound is sharp, which means that for $l\in\NN$ and $d\geq 2^l$, there exist CPA functions mapping from $\RR^d$ that cannot be represented using $l$ hidden layers.
While this is straightforward for $l=1$ \citep{mukherjee2017}, \citet{Hertrich2023TowardsLowerBounds} proved it for $l=2$ under assumptions on the functions represented by intermediate neurons, and \citet{haase2023lower} proved it for higher dimensions under an integer-weight assumption.

\autoref{thm:NN_representation_intro} falls within the 'shallow' regime described by the conjecture, which means that the depth is minimal and, in particular, independent of the complexity of the function.
The most general results of that type for neural network representations of CPA functions are given by \citet{Hertrich2023TowardsLowerBounds, Koutschan2023}.
In contrast, \citet{Hanin2019} introduced a narrow network architecture whose width depends only on $d$, but at the cost of significant depth.
\citet{Chen2022neurCompl} aimed to minimise the total number of neurons, making both the depth and width dependent on the complexity.
If the set of pieces satisfies a regularity condition, interpolation between a shallow and narrow configuration is possible \citep{brandenburg2024decompositionPolyhedra}. 

The bounds on the size of the representing architectures are often expressed in terms of the number of pieces $p$, with the additional requirement that these pieces be convex.
The reason is that this allows the direct application of the $\max$-$\min$ representation by \citet{Tarela1990MaxMin}.
If the pieces were allowed to be arbitrary sets, the number of pieces would match the number $n$ of unique affine functions needed to describe the CPA function.
\autoref{fig:pieces_vs_components} illustrates that $n$ and the number of convex pieces may not always serve as ideal complexity measures.
This work is the first to assume that the pieces are connected, without requiring them to be convex.

While previous results apply to arbitrary dimensions, in the specific case of $d=2$, none achieve a neuron count scaling linearly with the number of pieces.
The closest result to \autoref{thm:NN_representation_intro}, by \citet{Koutschan2023}, shows that a shallow network with width $O(n^{d+1})$ suffices to represent any CPA function in $\RR^d$, implying an $O(p^{d+1})$ bound also in terms of $p$. 
This paper improves this bound for $d=2$, reducing it from $O(p^3)$ to $O(p)$.

Representation properties of neural networks are crucial for quantitative function approximation results. 
A common approach is to use known approximation results for standard function classes and to emulate or approximate these classes efficiently with neural networks, thereby transferring their approximation rates (see, e.g., \citet{DeVore2021}).
In this context, \citet{He2020FEM} demonstrated that, for any fixed dimension, linear finite element spaces with convex-support nodal basis functions can be represented using $O(p)$ neurons. 
The constant in this representation depends on the shape-regularity of the triangulation. 
Similarly, \citet{HE2022Hierarchical} established an $O(p)$ bound for linear finite elements on uniform triangulations in $\RR^2$.
For univariate inputs, neural networks effectively serve as free-knot spline approximators \citep{Opschoor2020}. 
Thus, neural networks function as adaptive linear finite element approximations that are able to recover the mentioned classical counterparts with the same parameter count.

This work extends these findings by proving that a neural network with two inputs and $O(p)$ parameters can represent any CPA function with $p$ pieces, unrestricted by specific triangulations or meshes.

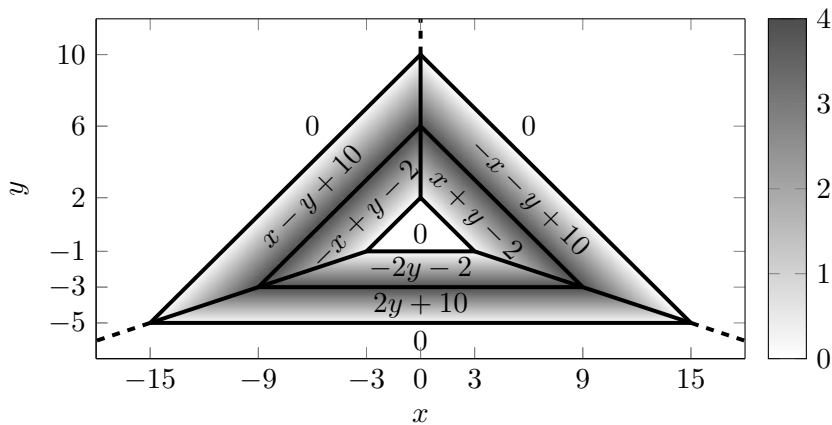
\begin{figure}[h]
	\centering
	\tikzsetnextfilename{pieces_vs_components}
\begin{tikzpicture}[very thick]
	\begin{axis}[
		view={0}{90},
		width=.8\textwidth,
		axis equal image,
		xmin=-18,xmax=18,ymin=-7,ymax=12,
		xtick={-15,-9,-3,0,3,9,15},
		ytick={-5,-3,-1,2,6,10},
		xlabel={$x$},
		ylabel={$y$},
		colormap={bw}{
        		gray(0cm)=(1);
       		gray(1cm)=(.3);
    		},
    	colorbar, %
		]

		\addplot3[
		surf, %
		shader=interp,	%
		domain=-18:18,
    		y domain=-7:12,
    		samples=70,
		] 
		{min( max(0,x+y-2,-x+y-2,-2*y-2), max(0,min(-x-y+10,x-y+10,2*y+10)) )};
		
		\coordinate (TI) at (axis cs: 0,2);		%
		\coordinate (TM) at (axis cs: 0,6);		%
		\coordinate (TO) at (axis cs: 0,10);		%
		\coordinate (TR) at (axis cs: 0,12);		%
		\coordinate (LI) at (axis cs: -3,-1);
		\coordinate (LM) at (axis cs: -9,-3);
		\coordinate (LO) at (axis cs: -15,-5);
		\coordinate (LR) at (axis cs: -18,-6);
		\coordinate (RI) at (axis cs: 3,-1);
		\coordinate (RM) at (axis cs: 9,-3);
		\coordinate (RO) at (axis cs: 15,-5);
		\coordinate (RR) at (axis cs: 18,-6);
		
		\draw[line width=1.5pt] (TI) -- (TO) ;
		\draw[line width=1.5pt] (LI) -- (LO) ;
		\draw[line width=1.5pt] (RI) -- (RO) ;
		\draw[line width=1.5pt, dashed] (TO) -- (TR) ;
		\draw[line width=1.5pt, dashed] (LO) -- (LR) ;
		\draw[line width=1.5pt, dashed] (RO) -- (RR) ;
		\draw[line width=1.5pt] (TI) -- (RI) -- (LI) -- cycle ;
		\draw[line width=1.5pt] (TM) -- (RM) -- (LM) -- cycle ;
		\draw[line width=1.5pt] (TO) -- (RO) -- (LO) -- cycle ;
	
		\node[] at (axis cs: 0,0) {$0$};
		\node[rotate=-45] 	at ($0.25*(TI)+0.25*(TM)+0.25*(RI)+0.25*(RM)$)	{$x+y-2$};    %
		\node[rotate=45] 	at ($0.25*(TI)+0.25*(TM)+0.25*(LI)+0.25*(LM)$) 	{$-x+y-2$};   %
		\node[] 			at ($0.25*(LI)+0.25*(LM)+0.25*(RI)+0.25*(RM)$)	{$-2y-2$};    %
		\node[rotate=-45] 	at ($0.25*(TM)+0.25*(TO)+0.25*(RM)+0.25*(RO)$)	{$-x-y+10$};  %
		\node[rotate=45] 	at ($0.25*(TM)+0.25*(TO)+0.25*(LM)+0.25*(LO)$) 	{$x-y+10$};   %
		\node[] 			at ($0.25*(LM)+0.25*(LO)+0.25*(RM)+0.25*(RO)$)	{$2y+10$};    %
		\node[] 			at (axis cs: 6,6)	{$0$};    %
		\node[] 			at (axis cs: 0,-6)	{$0$};    %
		\node[]				at (axis cs: -6,6)	{$0$};    %
		\end{axis}
\end{tikzpicture}                   
	\caption{A CPA function $f$ described locally by affine functions. 
		The labelled coordinates emphasise relevant positions.
		At least $p=8$ connected pieces are required to define $f$. 
		However, if restricted to convex pieces, at least $10$ are needed. 
		In contrast, $n=7$ suffice if the connectivity assumption is dropped. 
		Arguably, this function is less complex than a modified version with three convex outer regions (dashed lines) corresponding to different affine functions, but more complex than a simplified version where the inner triangle is removed by expanding adjacent regions. 
		The difference between the different counts can be made arbitrarily large by adding shifted copies of $f$.
		Asymptotic bounds on $p$ in terms of $n$ are given in \citet{zanotti2025pieces}.}
	\label{fig:pieces_vs_components}
\end{figure}

	\section{Definitions}\label{sec:Defs}
This section introduces the definitions and notation used throughout this work, with a primary focus on continuous piecewise affine functions.
Section \ref{sec:PolygonMotivation} provides the motivation for how I will define these functions in the subsequent sections \ref{sec:Polygons} and \ref{sec:CPA}.
Subsequently, the class of $\relu$ neural networks is introduced in section \ref{sec:NNs}.
For $n\in\NN$, I use the notation $[n]:=\set{1,\dots,n}$.
For a set $X\subset\RR^d$, $\interior{X}$ denotes the interior of $X$, and $\cl{X}$ denotes the closure of $X$.

\subsection{Motivation for a Non-Standard Definition of Polygons}\label{sec:PolygonMotivation}

Let us begin with the most general notion of a continuous piecewise affine function. Let $\cP$ be a partition of $\RR^2$, and let $f : \RR^2 \to \RR$ be continuous, such that for each $P \in \cP$, there exists an affine function $f_P$ with $f|_P = f_P$. For now, I call such a function continuous piecewise affine. However, this definition is too broad, as choosing $\cP$ as singletons would make every continuous function piecewise affine.

To avoid this and to use the number of sets in $\cP$ as a measure of complexity, impose the condition that $|\cP|$ is finite.

Additionally, we can assume $f_P \neq f_Q$ for distinct $P, Q \in \cP$ by merging sets where the same affine function applies. Then, the set of ambiguities $L:=\bigcup_{P\neq Q}\set{x:f_P(x)=f_Q(x)}$ consists of finitely many lines. Hence, the complement $U := \RR^2 \setminus L$ is a dense, open subset of $\RR^2$. By continuity of $f$, each $P^o := P \cap U$ is a union of connected components of $U$, with $U = \bigcup_{P \in \cP} P^o$. Thus, 
\begin{equation*}
	\bigcup_{P\in\cP}\cl{\interior{P}} \supset \bigcup_{P\in\cP}\cl{P^o} = \cl{\bigcup_{P\in\cP}P^o} = \cl{U}=\RR^2.
\end{equation*}
Continuity further ensures that $f|_{\cl{\interior{P}}} = f_P|_{\cl{\interior{P}}}$.

Consequently, by considering the closures of connected components of $\interior{P}$ as distinct sets, we can replace the partition $\cP$ with a finite cover of $\RR^2$, where each element is a regular closed set with connected interior. Moreover, the intersection of any two such sets is contained within a line.

This leads to the refined definition presented in the following sections.

\subsection{Polygons}\label{sec:Polygons}
Since I will directly work with the geometric structure of continuous piecewise affine functions, this structure will be defined very explicitly in a bottom-up manner.

Given $v,u\in\RR^2$ with $v\neq u$, I call $e:=\set{v+tu:\,t\in[0,1]}$ a \emph{line segment} connecting its \emph{vertices} $v$ and $v+u$. A line segment with vertices $a$ and $b$ will also be denoted by $\cl{ab}$. Similarly, I call $r:=\set{v+tu:\,t\geq 0}$ a \emph{ray} with initial vertex $v$, and $l:=\set{v+tu:\,t\in\RR}$ a \emph{line}. 
If $e\subset \RR^2$ is of one of these three types, its \emph{affine hull} is defined as the line containing $e$, i.e. if $v$ and $v+u$ are two distinct points on $e$, then $\aff(e):=\set{v+tu:\,t\in\RR}$ is its affine hull.

A \emph{polygonal arc} is just one line, or a subset of $\RR^2$ that is homeomorphic to a line, and which can be defined as the union of two rays and finitely many line segments, such that any two of the line segments and rays are disjoint or intersect at a common vertex. 
Similarly, I define a \emph{polygonal cycle} to be a subset of $\RR^2$ that is homeomorphic to the unit circle and which can be defined as the union of finitely many line segments, such that any two of the line segments are disjoint or intersect only at a vertex of both. 

The collection of line segments, rays, and lines, that were used in the definition of a polygonal arc or cycle $\gamma$, is called the \emph{edge} set of $\gamma$ and is denoted by $E(\gamma)$.
The set of \emph{vertices} $V(\gamma)$ of a polygonal arc or cycle $\gamma$ is defined as the union of the vertices of its edges.
These sets are not unique, since there are multiple ways to define the same curve.
Note that, due to the homeomorphisms, removing a point from $\gamma$ will locally disconnect the curve into exactly two parts.
Therefore, for every vertex $v\in V(\gamma)$, there are exactly two incident edges.

In this work, I adopt a more generalised notion of polygons than the standard definition. Specifically, a \emph{polygon} is defined as a closed subset $P\subseteq \RR^2$ with connected interior, and whose boundary $\partial P$ is the union of finitely many polygonal arcs and cycles, such that the intersection of any two of the arcs and cycles is either empty or a vertex of both curves. The corresponding arcs and cycles are called \emph{boundary components} of $P$.
Examples of polygons are shown in \autoref{fig:Polygon_examples}.
The set of line segments of a polygon $P$ is defined as the collection of the edges of all boundary components of $P$ that are line segments. This set will be denoted by $\Eb(P)$, where the subscript indicates that these are all the bounded edges of $P$. Similarly, the sets $E_r(P)$ and $E_l(P)$ are defined as the sets of rays and lines of $P$, i.e. the collections of rays respectively lines of the boundary components. Therefore, the set of \emph{edges} $E(P):=\Eb(P)\dot\cup E_r(P)\dot\cup E_l(P)$ completely characterises the boundary of $P$ via $\partial P=\bigcup_{e\in E(P)} e$.
Similarly, the set of vertices of $P$, denoted by $V(P)$, is the set of vertices of its boundary.
A point $x\in\RR^2$ is said to be in \emph{$P$-general position} if
\begin{equation*}%
	x\not\in \bigcup_{e\in E(P)} \aff(e).
\end{equation*} 

I note some basic facts about polygonal cycles.
Since a polygonal cycle is homeomorphic to a circle, it is a simple closed curve.
Due to the Jordan curve theorem, this means that any polygonal cycle $\gamma$ separates $\RR^2$ into two disjoint open connected components $\interior(\gamma)$ and $\exterior(\gamma)$, whose boundary is $\gamma$, such that $\interior(\gamma)$ is bounded and $\exterior(\gamma)$ unbounded. Therefore, if a cycle $\gamma$ is a boundary component of some polygon $P$, then the interior of $P$, which is connected, is completely contained in exactly one of these regions. If $P\subseteq\cl{\exterior(\gamma)}$, I call $\gamma$ a \emph{hole} of $P$. 

\begin{figure}[h]
    \centering
    \definecolor{colorBdrComp1}{HTML}{82B366}
\definecolor{colorBdrComp2}{HTML}{B85450}
\definecolor{colorBdrComp3}{HTML}{6C8EBF}

\begin{tabular}{C{4cm} C{4cm} C{4cm}}
     bounded & unbounded & \begin{tabular}{c} unbounded \\ with holes \end{tabular} \\ 
     \tikzsetnextfilename{LShape}
\begin{tikzpicture}[ultra thick]
    \coordinate (TL) at (0,2);
    \coordinate (TM) at (1,2);
    \coordinate (BL) at (0,0);
    \coordinate (MM) at (1,1);
    \coordinate (BR) at (2,0);
    \coordinate (MR) at (2,1);

    \draw[colorBdrComp1, fill_P_side, draw]%
        (BL) -- (BR) -- (MR) -- (MM) -- (TM) -- (TL) -- cycle;

    \fill (TL) circle (.05);
    \fill (TM) circle (.05);
    \fill (BL) circle (.05);
    \fill (MM) circle (.05);
    \fill (BR) circle (.05);
    \fill (MR) circle (.05);
    
\end{tikzpicture}
 & \tikzsetnextfilename{Polygon_unbounded}
\begin{tikzpicture}[ultra thick]
    \coordinate (lL) at (-1,1);
    \coordinate (lR) at (1,1);
    \coordinate (M) at (0,0);
    \coordinate (bL) at (-1,-1);
    \coordinate (bR) at (1,-1);

    \draw[fill_P_side] (lL) -- (lR) -- (bR) -- (M) -- (bL) -- cycle;

    \draw[colorBdrComp1, dashed] (lL)-- ($(lL)!.2!(lR)$);
    \draw[colorBdrComp1] %
        (lL) ($(lL)!.2!(lR)$) -- ($(lL)!.8!(lR)$);
    \draw[colorBdrComp1, dashed] (lL) ($(lL)!.2!(lR)$) -- (lR);
    
    \draw[colorBdrComp2] %
        (M) -- ($(M)!.6!(bL)$);
    \draw[colorBdrComp2, dashed] ($(M)!.6!(bL)$) -- (bL) ;
    \draw[colorBdrComp2] (M) -- ($(M)!.6!(bR)$);
    \draw[colorBdrComp2, dashed] ($(M)!.6!(bR)$) -- (bR) ;
    
    \fill (M) circle (.05);
    
\end{tikzpicture}
& \tikzsetnextfilename{polygon_unbd_holes}
\begin{tikzpicture}[ultra thick]
    \coordinate (lL) at (-1,1);
    \coordinate (lR) at (1.5,1);
    \coordinate (TL) at (-.5,.5);
    \coordinate (BL) at (-.5,-.5);
    \coordinate (M) at (0,0);
    \coordinate (TR) at (.5,.5);
    \coordinate (BR) at (.5,-.5);
    \coordinate (R) at (1,0);
    \coordinate (bL) at (-1,-1);
    \coordinate (bR) at (1.5,-1);

    \draw[fill_P_side] {(lL) -- (lR) [rounded corners]-- (bR) -- (bL) [sharp corners]-- cycle
    (TL) -- (BL) -- (M) -- cycle 
    (TR) -- (M) -- (BR) -- (R) -- cycle};

    \draw[colorBdrComp1, dashed] (lL)-- ($(lL)!.2!(lR)$);
    \draw[colorBdrComp1] %
        (lL) ($(lL)!.2!(lR)$) -- ($(lL)!.8!(lR)$);
    \draw[colorBdrComp1, dashed] (lL) ($(lL)!.2!(lR)$) -- (lR);
    \draw[colorBdrComp2]%
        (TL) -- (BL) -- (M) -- cycle;
    \draw[colorBdrComp3]%
        (TR) -- (R) -- (BR) -- (M) -- cycle;

    \fill (TL) circle (.05);
    \fill (BL) circle (.05);
    \fill (M)  circle (.05) node[below] {$v$};
    \fill (TR) circle (.05);
    \fill (R)  circle (.05);
    \fill (BR) circle (.05);
    
\end{tikzpicture}

\end{tabular}

    \caption{Examples of polygons according to my definition. Polygons are shown in grey, vertices by a dot, and the unbounded directions of lines and rays are indicated by dashed lines. The boundary of the left polygon consists of one polygonal cycle, the one in the middle of two polygonal arcs, and the right one consists of one polygonal arc and two polygonal cycles.}
    \label{fig:Polygon_examples}
\end{figure}
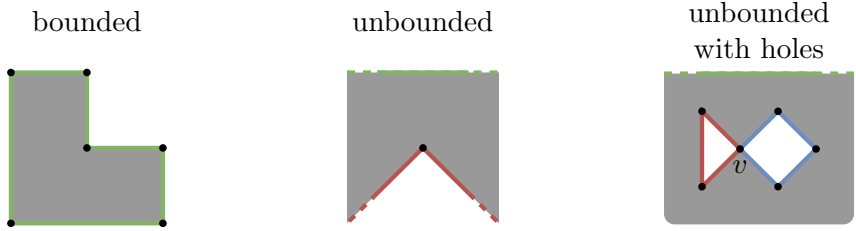

\subsection{Continuous Piecewise Affine Functions}\label{sec:CPA}
Let $f:\RR^2\to\RR$ be a continuous function. 
Let $\cP$ be a set of polygons that covers $\RR^2$, such that $P\cap Q=\partial P\cap \partial Q$ holds for any distinct $P,Q\in\cP$, and such that if $v\in V(P)$ for some $P\in\cP$, then $v\in V(Q)$ for all $Q\in\cP$ containing $v$.
Assume that there are affine functions $\set{f_P}_{P\in\cP}$ such that $f|_{P}=f_P$ for all $P\in\cP$. 
Then, $f$ is called \emph{continuous piecewise affine} (CPA), and $\cP$ is said to be an admissible set of \emph{pieces} for $f$.
The affine function $f_P$ satisfying $f|_{P}=f_P$ for a piece $P$ is called the \emph{affine component} of $f$ corresponding to $P$. 
The family of continuous piecewise affine functions in $\RR^2$ that can be defined by $p$ pieces is denoted by $\CPA_p$. 

Note that a piece can always be split into two which means that there is more than one set of admissible pieces and that $\CPA_p\subseteq \CPA_{p+1}$.
Moreover, the restriction on the vertices of the pieces ensures only that the pieces are mutually compatible, but still does not make the choice of vertices unique. 
Therefore, when discussing a $\CPA$ function, we must always fix a specific set of admissible pieces $\cP(f)$ along with their vertex and edge sets.

For a function $f\in\CPA_p$ with pieces $\cP(f)$, I carry over the names and definitions of the sets $V(P)$, $\Eb(P)$, $E_l(P)$ from polygons by taking the union of the respective sets over all the pieces of $f$. For example, the set of vertices of $f$ is defined by $V(f):=\bigcup_{P\in\cP(f)} V(P)$.
Note that because of $P\setminus\bigcup_{e\in E(f)} e=P\setminus\bigcup_{e\in E(P)} e=\interior{P}$, and because the interiors of distinct pieces are disjoint, we can write 
\begin{equation}\label{eq:CPA_expand_interior_indicator}
    f(x) = \sum_{P\in\cP(f)}\I{P}(x)\cdot f_P(x)\qquad \forall x\in\RR^2\setminus \bigcup_{e\in E(f)}e,
\end{equation}
where the indicator function $\I{A}$ for a set $A$ is defined by $\I{A}(x)=1$ if $x\in A$ and $\I{A}(x)=0$ otherwise.

If all affine components of a function $f\in\CPA_p$ are linear, i.e. $f_P(0)=0$ for every piece $P$, then I call the function $f$ \emph{continuous piecewise linear}, and denote the respective set of such functions by $\CPL_p\subsetneq\CPA_p$. 
Since two linear functions in $\RR^2$ only agree on a line through zero (or on all of $\RR^2$), the pieces of $f\in\CPL_p$ can always be chosen such that their boundary is the union of two rays with vertex $0$. 

\subsection{Neural Networks}\label{sec:NNs}
A second class of functions in this work are feed forward neural networks with the $\relu$ activation function $\rho(x):=\max(0,x)$, which is to be understood element-wise for inputs $x\in\RR^n$. 
Formally, I define a \emph{neural network} of \emph{depth} $L\in\NN$ as a sequence $s_0,\dots,s_{L}\in\NN$ together with a sequence of affine functions $T^{(l)}:\,\RR^{s_{l-1}}\to\RR^{s_l}$, for $l\in[L]$. Such a neural network is said to \emph{represent} the function 
\begin{equation*}
    f:\,\RR^{s_0}\to\RR^{s_L},\quad f = T^{(L)}\circ\rho\circ T^{(L-1)}\circ \dots\circ \rho\circ T^{(1)}.
\end{equation*}
For $l\in [L-1]$, $s_l$ is called the \emph{width} of the $l$-th \emph{hidden layer}. 
The \emph{width vector} of the network is defined as the vector $(s_l)_{l\in [L-1]}$.

I will only consider neural networks with $s_L=1$, since the results can be easily generalised to higher dimensional outputs.

	\section{Decomposition of CPA Functions}
The first step towards a standardised representation of $\CPA$ functions is to decompose them into a sum of simpler functions, each of which is associated with a vertex or edge of the considered function.

To achieve this, the pieces incident to a vertex or edge are extended to infinity using a definition similar to tangent cones for convex polyhedra.
For a polygon $P$ and an arbitrary point $v\in\RR^2$, I define the \emph{$P$-side} $Q^v_{P}$ of $v$ as follows. Let $D$ be a disk centered at $v$, small enough such that $D$ intersects only those edges of $P$ that have $v$ as a vertex. 
Then, define 
\begin{equation*}
	Q^v_{P} := \set{v+t(x-v): t\in[0,\infty),\, x\in P\cap D}.
\end{equation*}
This definition is independent of the choice of $D$ and satisfies the property that $Q^v_P$ is locally identical to $P$, in the sense that $P\cap D = Q^v_P\cap D$.
Note that the interior of $Q^v_P$ is not necessarily connected.
For an edge $e\in E(P)$, I define the \emph{$P$-side} $H^e_{P}$ of $e$ as a closed half-plane such that for any point $x\in e$, that is not a vertex of $e$, there is a small disk $D$, centered at $x$, such that $P\cap D = H^e_P\cap D$. This implies that $e\subseteq \partial H^e_P$.

For any subset $e\in\RR^2$ with $e\cap P=\emptyset$, I define its $P$-sides $Q^e_P$ and $H^e_P$ to be empty.

Let $f\in\CPA$ be a continuous piecewise affine function, and let $f_P$ denote the affine component corresponding to the piece $P\in\cP(f)$. For a vertex $v$ of $f$, I define the function $f^v$ piecewise by setting $f^v|_{Q_P^v}:=f_P$ for all $P\in\cP(f)$. 
For $x\in\RR^2$ that does not lie on the boundary of any $Q_P^v$, we can express $f^v$ as
\begin{equation}\label{eq:fv_expanded}
	f^v(x)=\sum_{P\in\cP(f)}\I{Q_P^v}(x)\cdot f_P(x).
\end{equation}
Note that, since the interior of $Q_P^v$ is not necessarily connected, the $P$-sides may not define valid pieces for a $\CPA$ function. However, their connected components do, and thus $f^v\in\CPA$. 
Moreover, the expression \eqref{eq:fv_expanded}, together with the local properties of the $P$-sides, implies that
\begin{equation*}
    f^v|_{D} = f|_{D},
\end{equation*}
for sufficiently small disks $D$ centered at $v$.

Similarly, I want to describe the local shape of $f$ at an edge by a global function.
Let $Q,R\in\cP(f)$ be two distinct pieces of $f$ that share an edge $e\in E(f)$, i.e. $e\subseteq Q\cap R$. Then, define $f^e\in\CPA_2$ on the two pieces $H_Q^e$ and $H_R^e$ by setting $f^e|_{H_Q^e}:=f_Q$ and $f^e|_{H_R^e}:=f_R$. Since $H_P^e=\emptyset$ for all $P\in\cP(f)\setminus\set{Q,R}$, it follows that
\begin{equation}\label{eq:fe_expanded}
	f^e(x) = \sum_{P\in\cP(f)}\I{H_P^e}(x)\cdot f_P(x)\quad\forall x\not\in \aff(e).
\end{equation}
Examples of $f^v$ and $f^e$ are shown in \autoref{fig:localCPL}.

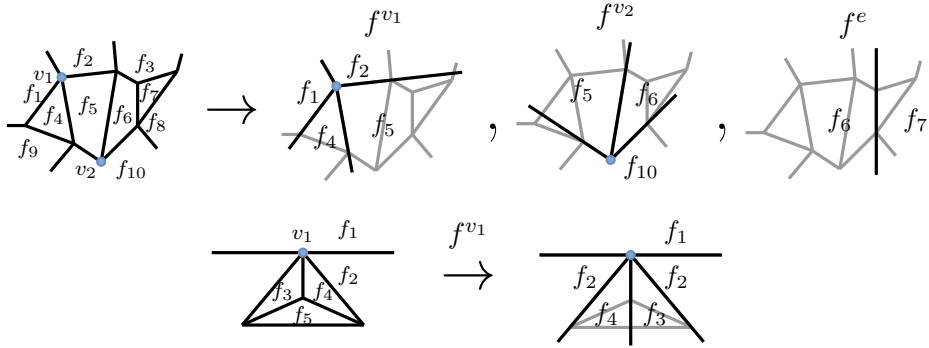
\begin{figure}[h]
    \centering
    \newcommand{\thL}{very thick}  %
\definecolor{colorVertexInner}{HTML}{7EA6E0}
\definecolor{colorVertexOut}{HTML}{6C8EBF}

\tikzsetnextfilename{CPA_building_blocks_1}
\begin{tikzpicture}[scale=.8]
    \coordinate (L) at (0,0);
    \coordinate (TL) at (.6,.8);
    \coordinate (BL) at (.8,-.3);
    \coordinate (T) at (1.5,.9);
    \coordinate (B) at (1.25,-.6);
    \coordinate (TR) at (1.85,.7);
    \coordinate (BR) at (1.85,0);
    \coordinate (R) at (2.5,.9);

    \coordinate (rL) at ($(L)-(0:.3)$);
    \coordinate (rTL) at ($(TL)+(120:.5)$);
    \coordinate (rBL) at ($(BL)+(230:.6)$);
    \coordinate (rT) at ($(T)+(95:.5)$);
    \coordinate (rBR) at ($(BR)+(-50:.5)$);
    \coordinate (rR) at ($(R)+(75:.3)$);
    
    \tikzset{
        skeleton/.pic={
            \draw    [\thL] (L) -- (TL) ;
            \draw    [\thL] (L) -- (BL) ;
            \draw    [\thL] (TL) -- (BL) ;
            \draw    [\thL] (TL) -- (T) ;
            \draw    [\thL] (BL) -- (B) ;
            \draw    [\thL] (T) -- (B) ;
            \draw    [\thL] (T) -- (TR) ;
            \draw    [\thL] (B) -- (BR) ;
            \draw    [\thL] (TR) -- (BR) ;
            \draw    [\thL] (TR) -- (R) ;
            \draw    [\thL] (BR) -- (R) ;
            \draw    [\thL] (L) -- (rL) ;
            \draw    [\thL] (TL) -- (rTL) ;
            \draw    [\thL] (BL) -- (rBL) ;
            \draw    [\thL] (T) -- (rT) ;
            \draw    [\thL] (BR) -- (rBR) ;
            \draw    [\thL] (R) -- (rR) ;
        }
    }

    \node(f){
        \begin{tikzpicture}
            \pic at (0,0) {
                skeleton
            };
        
            \fill[draw={colorVertexOut}, fill={colorVertexInner}, thick] (TL) circle (1.5pt) node[left, inner sep=2pt] {\scriptsize{$v_{1}$}};
            \fill[draw={colorVertexOut}, fill={colorVertexInner}, thick] (B) circle (1.5pt) node[below left, inner sep=1pt] {\scriptsize{$v_{2}$}};
            
            \node[] at ($0.3*(rL)+0.2*(L)+0.2*(TL)+0.3*(rTL)$) {\scriptsize{$f_1$}};
            \node[] at ($0.3*(rTL)+0.2*(TL)+0.2*(T)+0.3*(rT)$) {\scriptsize{$f_2$}};
            \node[] at ($0.2*(rT)+0.2*(T)+0.2*(TR)+0.2*(R)+0.2*(rR)$) {\scriptsize{$f_3$}};
            \node[] at ($1/3*(L)+1/3*(TL)+1/3*(BL)$) {\scriptsize{$f_4$}};
            \node[] at ($.3*(TL)+.2*(BL)+.2*(B)+.3*(T)$) {\scriptsize{$f_5$}};
            \node[] at ($.4*(B)+.6*(TR)$) {\scriptsize{$f_6$}};
            \node[] at ($0.4*(TR)+0.3*(BR)+0.3*(R)$) {\scriptsize{$f_7$}};
            \node[] at ($0.25*(R)+0.3*(BR)+0.45*(rBR)$) {\scriptsize{$f_8$}};
            \node[] at ($.5*(rL)+.5*(rBL)$) {\scriptsize{$f_9$}};
            \node[] at ($(B)+(.4,-.1)$) {\scriptsize{$f_{10}$}};
        \end{tikzpicture}
    };

    \node[at={($(f.east)$)}, anchor=west](pfeil){\huge$\rightarrow$};
    
    \node[at={($(f.east)$)}, shift={(1,0)},anchor=west] (fv1) {
        \begin{tikzpicture}
            \pic[Pcolor] at (0,0) {
                skeleton
            };

            \draw    [\thL] (TL) -- ($(TL)!2.3!(T)$) ;
            \draw    [\thL] (TL) -- ($(rTL)$) ;
            \draw    [\thL] (TL) -- ($(TL)!1.3!(L)$) ;
            \draw    [\thL] (TL) -- ($(TL)!1.3!(BL)$) ;
            
            \fill[draw={colorVertexOut}, fill={colorVertexInner}, thick] (TL) circle (1.5pt) node[left, inner sep=2pt] {};

            \node[anchor=center] at ($0.3*(rL)+0.1*(L)+0.1*(TL)+0.5*(rTL)$) {\small{$f_1$}};
            \node[anchor=center] at ($0.3*(rTL)+0.2*(TL)+0.2*(T)+0.3*(rT)$) {\small{$f_2$}};
            \node[anchor=center] at ($.45*(L)+.1*(TL)+.45*(BL)$) {\small{$f_4$}};
            \node[anchor=center] at ($0*(TL)+0*(BL)+.5*(B)+.5*(T)$) {\small{$f_5$}};
        
        \end{tikzpicture}
    };
    \node[at={($(fv1.north)+(.2,-.2)$)}, anchor=south]{$f^{v_1}$};

    \node[at={($(fv1.east)$)}, anchor=north west](komma1){\Huge ,};
    
    \node[at={($(fv1.east)$)}, shift={(.4,0)},anchor=west] (fv2) {
        \begin{tikzpicture}
            \pic[Pcolor] at (0,0) {
                skeleton
            };

            \draw    [\thL] (B) -- ($(B)!3!(BL)$) ;
            \draw    [\thL] (B) -- ($(B)!1.3!(T)$) ;
            \draw    [\thL] (B) -- ($(B)!1.8!(BR)$) ;

            \fill[draw={colorVertexOut}, fill={colorVertexInner}, thick] (B) circle (1.5pt) node[below left, inner sep=1pt] {};

            \node[anchor=center] at ($.7*(TL)+.1*(BL)+.1*(B)+.1*(T)$) {\small{$f_5$}};
            \node[anchor=center] at ($.3*(BR)+.7*(TR)$) {\small{$f_6$}};
            \node[anchor=center] at ($(B)+(.4,-.1)$) {\small{$f_{10}$}};
        
        \end{tikzpicture}
    };
    \node[at={($(fv2.north)+(.2,-.2)$)}, anchor=south]{$f^{v_2}$};

    \node[at={($(fv2.east)$)}, anchor=north west](komma2){\Huge,};
    
    \node[at={($(fv2.east)$)}, shift={(.4,0)}, anchor=west] (fe) {
        \begin{tikzpicture}
            \pic[Pcolor] at (0,0) {
                skeleton
            };

            \draw    [\thL] ($(BR)!2!(TR)$) -- ($(TR)!2!(BR)$) ;

            \node[anchor=center] (f6) at ($.1*(L)+.5*(T)+.4*(B)$) {\small{$f_6$}};
            \node[anchor=center] at ($(f6)+(1,0)$) {\small{$f_7$}};
        
        \end{tikzpicture}
    };
    \node[at={($(fe.north)+(.2,-.2)$)}, anchor=south]{$f^{e}$};
    
\end{tikzpicture}\\
    \definecolor{colorVertexInner}{HTML}{7EA6E0}
\definecolor{colorVertexOut}{HTML}{6C8EBF}

\tikzsetnextfilename{CPA_building_blocks_2}
\begin{tikzpicture}[scale=.8]
	\coordinate (T) at (0,0);
	\coordinate (M) at (0,-.75);
	\coordinate (L) at (-1,-1.2);
	\coordinate (R) at (1,-1.2);
	
	\coordinate (rL) at (-1.5,0);
	\coordinate (rR) at (1.5,0);
	
	\tikzset{
		skeleton/.pic={
			\draw    [\thL] (T) -- (L) ;
			\draw    [\thL] (L) -- (R) ;
			\draw    [\thL] (R) -- (T) ;
			\draw    [\thL] (M) -- (T) ;
			\draw    [\thL] (M) -- (R) ;
			\draw    [\thL] (M) -- (L) ;
			\draw    [\thL] (T) -- (rL) ;
			\draw    [\thL] (T) -- (rR) ;
		}
	}
	
	\node(f){
		\begin{tikzpicture}
			\pic at (0,0) {
				skeleton
			};
			
			\fill[draw={colorVertexOut}, fill={colorVertexInner}, thick] (T) circle (1.5pt) node[above] {\scriptsize{$v_{1}$}};
			
			\node[] at (.6,.3) {\scriptsize{$f_1$}};
			\node[] at (.6,-.3) {\scriptsize{$f_2$}};
			\node[] at ($1/3*(T)+1/3*(L)+1/3*(M)$) {\scriptsize{$f_3$}};
			\node[] at ($1/3*(T)+1/3*(R)+1/3*(M)$) {\scriptsize{$f_4$}};
			\node[] at ($1/3*(L)+1/3*(R)+1/3*(M)$) {\scriptsize{$f_5$}};
		\end{tikzpicture}
	};
	
	\node[at={($(f.east)+(.4,-.1)$)}, anchor=west](pfeil){\huge$\rightarrow$};
	
	\node[at={($(f.north east)+(2,0)$)},anchor=north west] (fv1) {
		\begin{tikzpicture}
			\coordinate (rbL) at ($(T)!1.2!(L)$);
			\coordinate (rbR) at ($(T)!1.2!(R)$);
			\coordinate (rbM) at ($(T)!2.0!(M)$);
			
			\pic[Pcolor] at (0,0) {
				skeleton
			};
			
			\draw    [\thL] (T) -- (rL) ;
			\draw    [\thL] (T) -- (rR) ;
			\draw    [\thL] (T) -- (rbL);
			\draw    [\thL] (T) -- (rbM) ;
			\draw    [\thL] (T) -- (rbR) ;
			
			\fill[draw={colorVertexOut}, fill={colorVertexInner}, thick] (T) circle (1.5pt) node[left, inner sep=2pt] {};
			
			\node[anchor=center] at (.6,.3) {\small{$f_1$}};
			\node[anchor=center] at (.6,-.3) {\small{$f_2$}};
			\node[anchor=center] at (-.6,-.3) {\small{$f_2$}};
			\node[anchor=center] at ($1/3*(T)+1/3*(rbM)+1/3*(rbR)$) {\small{$f_3$}};
			\node[anchor=center] at ($1/3*(T)+1/3*(rbM)+1/3*(rbL)$) {\small{$f_4$}};

		\end{tikzpicture}
	};
	
	\node[at={($(pfeil.north)$)}, anchor=south]{$f^{v_1}$};

\end{tikzpicture}
    \caption{Two $\CPA$ functions and some of the building blocks that will be used in their decomposition. In the upper example, $e$ is the edge between the pieces $P_6$ and $P_7$. The example below shows $f^{v_1}$ for a vertex $v_1$ whose $P_2$-side has disconnected interior, resulting in two separate pieces.}
    \label{fig:localCPL}
\end{figure}

For a polygon $P$, let $n_h(P)$ and $n_a(P)$ denote the number of holes respectively arcs that are boundary components of $P$. 
With ${\deg_P(v):=\abs{\set{e\in E(P):\, v\in e}}}$, the number $d(P):=\sum_{v\in V(P)}(\deg_P(v)/2-1)$ is only non-zero if there is a vertex in which at least two boundary components intersect. 
Finally, define $c(P):=1+d(P)-n_h(P)-n_a(P)$.

The following lemma describes the function $f$ using only the building blocks $f^v$, $f^e$, and a linear combination of its affine components.
Note that a similar result can be derived from \cite[Prop. 18]{Tran2024MinimalTRF}, but only up to an affine function.

\begin{lemma}%
\label{lem:CPA_Decomp}
Let $f\in\CPA$ be a continuous piecewise affine function in $\RR^2$. Denote the affine component corresponding to $P\in\cP(f)$ by $f_P$.
Then, it holds that
\begin{equation}\label{eq:Decomposition}
    f(x) = \sum_{v\in V(f)} f^{{v}}(x) + \sum_{\substack{e\in E_l(f)}}f^{{e}}(x) - \sum_{\substack{e\in \Eb(f)}}f^{{e}}(x) +\sum_{P\in\cP(f)}c(P)f_P(x).
\end{equation}
\end{lemma}

For the proof of \autoref{lem:CPA_Decomp}, the following lemma, proved in \autoref{sec:InsideOutsideProof}, is essential.
This lemma extends the conic decomposition of convex polytopes (see, e.g., \citet{Shephard1967AngleSums}) to my definition of polygons.
\begin{lemma}%
\label{lem:InsideOutside}
    Let $P$ be a polygon and let $x\in\RR^2$ be an arbitrary point in $P$-general position.
    Then, it holds that
    \begin{equation*}%
        \sum_{v\in V(P)}\I{Q^v_P}(x)+\sum_{e\in E_l(P)}\I{H^e_P}(x)-\sum_{e\in \Eb(P)}\I{H^e_P}(x)+c(P)
        = \I{P}(x).
    \end{equation*}
\end{lemma}

\begin{proof}[Proof of \autoref{lem:CPA_Decomp}]
First, note that it is sufficient to show \eqref{eq:Decomposition} for $x\in\RR^2\setminus\bigcup_{e\in E(f)}\aff(e)$, since both sides describe continuous functions. 
As $\bigcup_{e\in E(f)}\aff(e)$ contains all edges of the functions $f$, $f^v$, and $f^e$, we can apply \eqref{eq:CPA_expand_interior_indicator}, \eqref{eq:fv_expanded}, and \eqref{eq:fe_expanded} to express these functions in an analytic form.
The left hand side of \eqref{eq:Decomposition} becomes
\begin{align} \label{eq:expand_lhs}
    f(x)&= \sum_{P\in\cP(f)}\I{P}(x)\cdot f_P(x),
\end{align}
and the right hand side becomes
\begin{align} \label{eq:expand_rhs}
    \sum_{v\in V(f)} \hspace{-1.4em}&\hspace{1.4em}
    f^{v}(x) + \sum_{\substack{e\in E_l(f)}}f^{e}(x) - \sum_{\substack{e\in \Eb(f)}}f^{e}(x) +\sum_{P\in\cP(f)}c(P)f_P(x) \nonumber \\
     &= \sum_{v\in V(f)} \sum_{P\in\cP(f)}\ones_{Q^v_P}(x) f_P(x) + \sum_{\substack{e\in E_l(f)}} \sum_{P\in\cP(f)}\ones_{H^e_P}(x) f_P(x) \nonumber \\
     &\qquad- \sum_{\substack{e\in \Eb(f)}} \sum_{P\in\cP(f)}\ones_{H^e_P}(x) f_P(x)  +\sum_{P\in\cP(f)}c(P)f_P(x) \nonumber \\
     &= \sum_{P\in\cP(f)} 
     \hspace{-.3em} 
     f_P(x)\underbrace{\bigg( \sum_{v\in V(f)} \ones_{Q^v_P}(x) + 
     \hspace{-.3em}
     \sum_{\substack{e\in E_l(f)}} \ones_{H^e_P}(x) - 
     \hspace{-.3em}
     \sum_{\substack{e\in \Eb(f)}} \ones_{H^e_P}(x) + c(P)\bigg)}_{=:\,\mytag{\ensuremath{I_P(x)}}{tag:PiecesOverlap}}
\end{align}
Comparing \eqref{eq:expand_lhs} and \eqref{eq:expand_rhs}, we see that \eqref{eq:Decomposition} is true, if $\ref{tag:PiecesOverlap}=\ones_{P}(x)$ for all $P\in\cP(f)$.
Using that 
\begin{equation*}%
	Q_P^v\neq\emptyset\quad\Leftrightarrow\quad v\in V(P) \qquad\text{and}\qquad H_P^e\neq\emptyset\quad\Leftrightarrow\quad e\in E(P),
\end{equation*}
$I_P(x)$ can be written as
\begin{align*}
    I_P(x) 
    & =\sum_{v\in V(P)}\hspace{-.2em} \ones_{Q^v_P}(x) + \hspace{-.1em}\sum_{\substack{e\in E_l(P)}}\hspace{-.2em} \ones_{H^e_P}(x) - \hspace{-.1em}\sum_{\substack{e\in \Eb(P)}}\hspace{-.2em} \ones_{H^e_P}(x) + c(P)
\end{align*}
Since we have restricted ourselves to $x$ in $P$-general position, we can apply \autoref{lem:InsideOutside}, which directly yields $\ref{tag:PiecesOverlap}=\I{P}(x)$.
\end{proof}

\begin{remark}
    Let $G=(V,E)$ be a connected plane graph such that its set of faces $\cP$ consists of polygons.
    For $f(x):=1$, which is $\CPA$ on any admissible set of pieces, \autoref{lem:CPA_Decomp} implies that 
    \begin{equation*}%
    	1 = \sum_{v\in V} 1 - \sum_{e\in E} 1 + \sum_{P\in\cP} (1+d(P)-n_h(P)).
    \end{equation*}
    Since $G$ is connected, there is one outer face, say $P_0\in\cP$, whose boundary is a hole, i.e. $d(P_0)=0$ and $n_h(P_0)=1$.
    For all the other faces $P\in\cP\setminus\set{P_0}$, every hole must intersect the boundary $\partial P$, and there is exactly one intersection, since otherwise the hole would make $P$ disconnected. Therefore, $d(P)=n_h(P)$ for all such $P$.
    In total, \begin{equation*}%
    	2 = |V| - |E| + |\cP|,
    \end{equation*}
    which is Euler's formula.
\end{remark}

Next is the proof of \autoref{lem:InsideOutside}.

\subsection{Proof of \autoref{lem:InsideOutside}}\label{sec:InsideOutsideProof}
Since my definition of polygons allows for very complicated shapes, I will conduct the proof of \autoref{lem:InsideOutside} in multiple steps. 
I begin in section \ref{sec:inside_outside_cycle} by proving the lemma for polygons whose only boundary component is a single cycle. 
This result will then be utilised in section \ref{sec:inside_outside_arcs}, where I consider the case in which the boundary does not contain cycles.
By viewing a general polygon as an outer polygon with some holes cut out, I will combine the two special cases to prove the complete statement in section \ref{sec:inside_outside_final}.

\subsubsection{Only one cycle}\label{sec:inside_outside_cycle}
I start with a special case of \autoref{lem:InsideOutside}.

\begin{lemma}%
\label{lem:InsideOutsideBounded}
    Let $P$ be a polygon whose boundary consists of a single polygonal cycle $\gamma$, and let $x\in\RR^2$ be an arbitrary point in $P$-general position.
    Then, it holds that
    \begin{equation}\label{eq:inside_outside_lem_bd}
        \sum_{v\in V(P)}\I{Q^v_P}(x)-\sum_{e\in \Eb(P)}\I{H^e_P}(x)-n_h(P)
        = \I{P}(x)-1.
    \end{equation}
\end{lemma}

Note that in this case $n_h(P)$ is $1$ if $P$ is unbounded, i.e. if $\gamma$ describes a hole, and $0$ otherwise. 
Moreover, we have $\deg_P(v)=2$ for all $v\in V(P)$, and thus $d(P)=0$. As the boundary does not contain any arcs, this implies $c(P)=1-n_h(P)$. Therefore, \autoref{lem:InsideOutsideBounded} is indeed a special case of \autoref{lem:InsideOutside} with $E_l(P)=\emptyset$.

\begin{proof}
Let us first assume that $P=\cl{\interior(\gamma)}$, which is equivalent to $n_h(P)=0$.
Define $v_P(x):=\sum_{v\in V(P)}\I{Q^v_P}(x)$, and $e_P(x):=\sum_{e\in \Eb(P)}\I{H^e_P}(x)$. 
Note that a sum over indicators corresponds to counting the number of sets containing $x$. 
Therefore, we can reformulate the left hand side of \eqref{eq:inside_outside_lem_bd} as
\begin{align}
    v_P(x)-e_P(x)
    &= \abs{\set{v\in V(P): x\in Q^v_P}}-\abs{\set{e\in \Eb(P): x\in H^e_P}},
    \label{eq:inside_outside_bounded_tmp}
\end{align}
which is the difference between the number of vertices and edges that have $x$ on their respective $P$-side.

Let $n:=|V(P)|$ be the number of vertices of $P$. 
In the following, indices are always to be understood modulo $n$.
Since $P$ is homeomorphic to a closed disk, we can enumerate its vertices $V(P)=\set{v_k}_{k=1}^n$ and edges $\Eb(P)=\set{e_k}_{k=1}^n$ in a cyclic order such that $e_k=\overline{v_{k-1}v_{k}}$.
Moreover, by the Jordan–Schönflies theorem, we can choose the order in a counterclockwise sense, such that $P$ is always on the left hand side when traversing $e_k$ from $v_{k-1}$ to $v_k$.
We define $H^{k}_+:=H^{e_k}_P$ to be the $P$-side of $e_k$, and $H^{k}_-$ as the opposite half-plane. Further, $Q^k:=Q^{v_k}_P$ is short for the $P$-side of $v_k$. With this, \eqref{eq:inside_outside_bounded_tmp} reads
\begin{equation}\label{eq:inside_outside_bounded}
    v_P(x)-e_P(x) = |\set{k\in[n]:x\in Q^k}|-|\set{k\in[n]:x\in H^{k}_+}|.
\end{equation}

From now on, we assume $x\in\RR^2$ in $P$-general position to be fixed.
Whether $x$ lies on the $P$-side of a vertex is fully determined by the incident edges. 
From \autoref{fig:effects}, we see that if the corner formed by the edges incident with vertex $v_k$ is concave, $x\in Q^k$ is equivalent to $x\in H^k_+\cup H^{k+1}_+$. If it is convex, then $x\in Q^k$ if and only if $x\in H^k_+\cap H^{k+1}_+$.
This is summarised in \autoref{tab:sign_convexity_cases}.
In the following, I want to refer to the eight different cases listed in \autoref{tab:sign_convexity_cases} in a short form. For example, I write $k=\texttt{+-convex}$ if $x\in H^k_+\cap H^{k+1}_-$ and the corner at $v_k$ is convex, and $k=\texttt{-+concave}$ if $x\in H^k_-\cap H^{k+1}_+$ and the corner is concave.

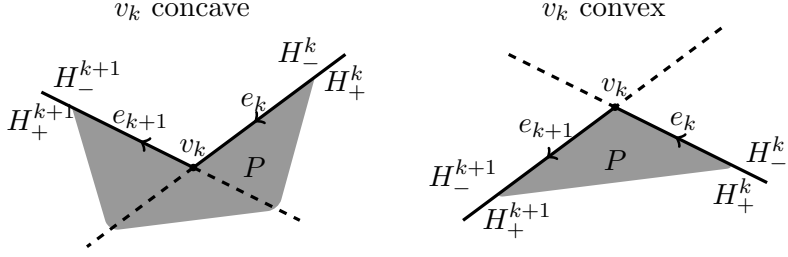
\begin{figure}[h]
    \centering
    \def\linc{.8}

\begin{tabular}{cc}
         $v_k$ concave & $v_k$ convex \\
\tikzsetnextfilename{P_side_vertex_concave}
\begin{tikzpicture}[scale=.1, very thick]
    \coordinate (L1a) at ($.7*(-20,-15)$);
    \coordinate (L1b) at (20,15);
    \coordinate (L2a) at (-20,10);
    \coordinate (L2b) at ($.7*(20,-10)$);
    \coordinate (O) at (0,0);
    \draw[ fill_P_side ] ($(O)!\linc!(L1b)$) -- (O) -- ($(O)!\linc!(L2a)$) [rounded corners]-- ($(O)!\linc!(L1a)$) -- ($(O)!\linc!(L2b)$) [sharp corners]-- cycle;

    \draw[postaction={decorate,decoration={markings,mark=at position 0.6 with {\arrow{>}; \node[above] {$e_k$};}}}] 
        (L1b) node[below] {$H^k_+$} node[left, xshift=-5] {$H^k_-$}  --  (O);
    \draw[dashed] (O) -- (L1a) ;
    \draw[postaction={decorate,decoration={markings,mark=at position 0.35 with {\arrow{>}; \node[above] {$e_{k+1}$};}}}] 
        (O) -- (L2a) node[below] {$H^{k+1}_+$} node[right, yshift=5] {$H^{k+1}_-$} ;
    \draw[dashed] (O) --  (L2b);

    \fill (O) circle (.5) node[above] {$v_k$};

	\node[] at (8,0) {$P$};

\end{tikzpicture}
&
\tikzsetnextfilename{P_side_vertex_convex}
\begin{tikzpicture}[scale=.1, very thick]
    \coordinate (L1a) at (-20,-15);
    \coordinate (L1b) at (20,15);
    \coordinate (L2a) at (-20,10);
    \coordinate (L2b) at (20,-10);
    \coordinate (O) at (0,0);
    
    \draw[fill_P_side] (O) [rounded corners] -- ($(O)!\linc!(L1a)$) -- ($(O)!\linc!(L2b)$) [sharp corners] -- cycle;
    \draw[dashed] ($.7*(L1b)$) -- (O);
    \draw[postaction={decorate,decoration={markings,mark=at position 0.45 with {\arrow{>}; \node[above, yshift=3] {$e_{k+1}$};}}}]
        (O) -- ($1*(L1a)$) node[right,xshift=3] {$H^{k+1}_+$} node[above,yshift=5] {$H^{k+1}_-$};
    \draw[dashed] ($.7*(L2a)$) -- (O);
    \draw[postaction={decorate,decoration={markings,mark=at position 0.45 with {\arrow{<}; \node[above] {$e_{k}$};}}}]
        (O) --  ($1*(L2b)$) node[left,yshift=-5] {$H^k_+$} node[above] {$H^k_-$};

    \fill (O) circle (.5) node[above] {$v_k$};

    \node[] at (0,-7) {$P$};

\end{tikzpicture}
\end{tabular}
    \caption{The $P$-side $Q^k$ of $v_k$ is indicated in grey.}
    \label{fig:effects}
\end{figure}

\begin{table}[h]
    \centering
    \begin{tabular}{ccc|c}
         $H^k_\bullet$ & $H^{k+1}_\bullet$ & Convexity at $v_k$ & \begin{tabular}{c}vertex \\ $x\in Q^k$ \end{tabular} \\ \hline
         $+$ & $+$ & convex & $\checkmark$   \\
         $+$ & $+$ & concave & $\checkmark$   \\
         $+$ & $-$ & convex & -  \\
         $+$ & $-$ & concave & $\checkmark$  \\
         $-$ & $+$ & convex & -  \\
         $-$ & $+$ & concave & $\checkmark$  \\
         $-$ & $-$ & convex & -  \\
         $-$ & $-$ & concave & -  \\
    \end{tabular}
    \caption{Summary of whether $x$ is on the $P$-side of vertex $v_k$, depending on the position of $x$ w.r.t. the incident edges and the convexity of the corresponding corner. The signs $+$ and $-$ in the first two columns mean $x\in H_+$ and $x\in H_-$ respectively.}
    \label{tab:sign_convexity_cases}
\end{table}

To verify \eqref{eq:inside_outside_lem_bd}, we consider each pair $\set{(v_k,e_{k+1})}_{k=1}^n$ of a vertex and the succeeding edge separately. As the contribution of $(v_k,e_{k+1})$ to \eqref{eq:inside_outside_bounded} cancels if $x\in Q^k\cap H^{k+1}_+$ or $x\not\in Q^k\cup H^{k+1}_+$, we get
\begin{equation*}
    v_P(x)-e_P(x)=|\set{k\in[n]:x\in Q^k\setminus H^{k+1}_+}|-|\set{k\in[n]:x\in H^{k+1}_+\setminus Q^k}|.
\end{equation*}
Considering \autoref{tab:sign_convexity_cases}, this is precisely
\begin{multline}\label{eq:inside_outside_bounded_simplified}
    v_P(x)-e_P(x) 
    \\= \abs{\set{k\in[n]: k=\texttt{+-concave}}}-\abs{\set{k\in[n]:k=\texttt{-+convex}}}.
\end{multline}

To determine \eqref{eq:inside_outside_bounded_simplified}, let us define certain angles, see \autoref{fig:angle_relations}, that help to distinguish between the different cases:
\begin{itemize}
    \item $\alpha_k:=\angle v_{k-1}xv_{k}\in(-\pi,\pi)$, where $\angle abc$ is to be understood as the angle from $a$ to $c$ about $b$. We have
   \begin{equation}\alpha_k\begin{cases} \label{eq:alpha_sign}
            >0,\quad x\in H^k_+ \\
            <0,\quad x\in H^k_- 
    \end{cases}
    \end{equation}
    because $P$ is always on the left when traversing $e_k$ from $v_{k-1}$ to $v_k$.    
    The sum of these angles is the winding number of $\gamma$ with respect to $x$, see, e.g., \citet{Hormann2001incremental}.
    Because $P = \cl{\interior(\gamma)}$ is the bounded component defined by the simple (i.e., non-self-intersecting) closed curve $\gamma$, and due to the counterclockwise enumeration, we can use the winding number to distinguish between the interior and exterior of $P$ (see, e.g., \citet[Section~5-7]{Carmo1976Differentialgeometry}):
    \begin{equation}\label{eq:alpha_sum_property}
        \sum_{k=1}^n \alpha_k = 2\pi\cdot \I{P}(x)
    \end{equation}
    \item $\eta_k:=\angle v_{k+1}v_kv_{k-1}\in(0,2\pi)$, measured counterclockwise. The corner at $v_k$ is convex if $\eta_k \leq\pi$ and concave if $\eta_k\geq\pi$. Further, since $P$ is a simple polygon in the classical sense, we have the relation
    \begin{equation}\label{eq:eta_sum_property}
        \sum_{k=1}^n \eta_k = (n-2)\pi
    \end{equation}
    \item $\beta_k:=\angle xv_kv_{k+1}\in (-\pi,\pi)$.
    As with $\alpha_k$, the sign of $\beta_k$ is related to the $P$-side of $e_{k+1}$, as $P$ is always on the left hand side:
    \begin{equation}\label{eq:beta_sign}
        \beta_k\begin{cases}
            <0,\quad x\in H^{k+1}_+ \\
            >0,\quad x\in H^{k+1}_- 
        \end{cases}
    \end{equation}
\end{itemize}
Note that $\alpha_k$ and $\beta_k$ are measured with values in $(-\pi,\pi)$, but $\eta_k$ with values in $(0,2\pi)$. Moreover, $\alpha_k$ and $\beta_k$ are non-zero since $x$ is in $P$-general position.

To determine the frequency of occurrences of the cases \texttt{+-concave} and \texttt{-+convex}, we consider, for each vertex-edge pair $(v_k,e_{k+1})$, the triangle defined by the points $x$, $v_{k-1}$ and $v_{k}$, see \autoref{fig:angle_relations}. 
These three points are not collinear since $x$ is in $P$-general position. Therefore, the sum of the triangle's interior angles satisfies the relation
\begin{equation*}
    \abs{\alpha_k} + \abs{\beta_{k-1}} + \delta_k = \pi,
\end{equation*}
where $\delta_k:=\abs{\angle v_{k-1}v_{k}x}\in (0,\pi)$.
Using the properties \eqref{eq:alpha_sign} and \eqref{eq:beta_sign}, this is equivalent to
\begin{equation} \label{eq:triangle_angles_eq_pi}
    \begin{rcases}
            \alpha_k-\beta_{k-1}, &x\in H^k_+ \\
            \beta_{k-1}-\alpha_k, &x\in H^k_- 
    \end{rcases} + \delta_k = \pi
\end{equation}

From the relations of the angles involving vertex $v_k$, we will now express $\delta_k$ in terms of $\beta_k$ and $\eta_k$. Since the orientation of $\beta_k$, and the overlap of the angles may change depending on the convexity of the corner, and on the relative position of $x$ with respect to $e_k$ and $e_{k+1}$, we need to consider the eight different cases from \autoref{tab:sign_convexity_cases} separately.
If we consider, for example, the case $k=\texttt{-+convex}$ in \autoref{fig:angle_relations}, we have $\delta_k+\eta_k=|\beta_k|$. Applying \eqref{eq:beta_sign}, we get $\delta_k=-\eta_k-\beta_k$, which, when substituted into \eqref{eq:triangle_angles_eq_pi}, implies $\beta_{k-1}-\beta_k=\alpha_k+\eta_k+\pi$.
In general, the difference between two consecutive $\beta_k$ can be summarised for all eight cases as 
\begin{equation}\label{eq:diff_consec_beta}
    \beta_{k-1}-\beta_{k} = \alpha_k+\eta_k-\pi+\begin{cases}
        \phantom{+}2\pi,\quad &k=\texttt{-+convex} \\
        -2\pi,\quad &k=\texttt{+-concave} \\
        \phantom{-2}0, &\text{else}
    \end{cases}
\end{equation}

Using properties \eqref{eq:alpha_sum_property} and \eqref{eq:eta_sum_property} of the angles $\alpha_k$ and $\eta_k$, \eqref{eq:diff_consec_beta} directly implies
\begin{align*}
    0 & = \sum_{k=1}^n\beta_{k-1}-\beta_{k} \\
    & = 2\pi\cdot\I{P}(x)- 2\pi + 2\pi\abs{\set{k: \texttt{-+convex}}} - 2\pi\abs{\set{k: \texttt{+-concave}}},
\end{align*}
and thus 
\begin{equation*}%
    \abs{\set{k: \texttt{+-concave}}} - \abs{\set{k: \texttt{-+convex}}} = \I{P}(x)-1.
\end{equation*}
By \eqref{eq:inside_outside_bounded_simplified}, this is precisely \eqref{eq:inside_outside_lem_bd} with $n_h(P)=0$.

\newcommand{\cint}{black}
\newcommand{\cobs}{red}
\newcommand{\csub}{blue}
\definecolor{gammaColor}{gray}{0.3}
\newcommand{\cgam}{gammaColor}
\newcommand{\dsz}{1.5pt}
\newcommand{\ki}{k}  %
\newcommand{\thL}{very thick}  %
\newcommand{\thA}{very thick}  %
\begin{figure}
    \centering
    
    \begin{tabular}{C{.7cm} C{.8cm}|C{4.5cm} C{4.5cm}}
         $H^k_\bullet$ & $H^{k+1}_\bullet$ & convex ($\eta_k<\pi$) & concave ($\eta_k>\pi$) \\ \hline
         \huge{$+$} & \huge{$+$} & \tikzsetnextfilename{ppConvex}
\begin{tikzpicture}[scale=1]
    \coordinate (X) at (7,-.3);
    \coordinate (A) at (6,2);
    \coordinate (B) at (4,0);
    \coordinate (C) at (5,-1);
    \coordinate (C2) at (5,-1);
    \coordinate (D) at (6,-1);
    \coordinate (E) at (6.5,2);
    \draw[fill_P_side_light] (A) -- (B) -- (C) [rounded corners]-- (D) --(E) [sharp corners]-- cycle;
    \draw[\thL, postaction={decorate,decoration={markings,mark=at position 0.5 with {\arrow{>}}}}] 
        (A) -- (B) node[pos=.5, above left, inner sep=1pt] {$e_k$};
    \draw[\thL] (B) -- (C) node[pos=.8, left] {$e_{k+1}$};
    \draw[\thL, \cobs] (X) -- (B);
    \draw[\thL, \cobs] (X) -- (A);
    \pic [\thA, "$\eta_\ki$", draw, angle radius=.75cm, \cint] {angle=C--B--A}; %
    \pic [\thA, "$\beta_{\ki}$", draw, <-, angle radius=1cm, angle eccentricity=1.5, \cobs] {angle=C--B--X}; 
    \pic [\thA, "$\beta_{\ki-1}$", draw, <-, angle radius=.8cm, angle eccentricity=1.3, \cobs] {angle=B--A--X}; 
    \pic [\thA, "$\delta_\ki$", draw, angle radius=1cm, angle eccentricity=1.3, \cgam] {angle=X--B--A}; 
    \pic [\thA, "$\alpha_\ki$", draw, ->, angle radius=.8cm, \csub] {angle=A--X--B}; 

    \fill (A) circle (\dsz) node[right] {$v_{\ki-1}$};
    \fill (B) circle (\dsz) node[left] {$v_{\ki}$};
    \fill (X) circle (\dsz) node[below] {$x$};

\end{tikzpicture} & \tikzsetnextfilename{ppConcave}
\begin{tikzpicture}[scale=1]
    \coordinate (X) at (5.5,1.2);
    \coordinate (A) at (2.2,1.8);
    \coordinate (B) at (4,0);
    \coordinate (C) at (3.7,-1.2);
    \coordinate (D) at (4.8,-1.2);
    \coordinate (E) at (5,0);
    \coordinate (F) at (4.5,1.8);
    \draw[fill_P_side_light] (A) -- (B) -- (C) [rounded corners]-- (D) -- (E) -- (F) [sharp corners]-- cycle;
    \draw[\thL, postaction={decorate,decoration={markings,mark=at position 0.6 with {\arrow{>}}}}] 
        (A) -- (B) node[pos=.5, below left, inner sep=1pt] {$e_{k}$};
    \draw[\thL] (B) -- (C) node[pos=.8, left] {$e_{k+1}$};
    \draw[\thL, \cobs] (X) -- (B);
    \draw[\thL, \cobs] (X) -- (A);
    \pic [\thA, "$\eta_\ki$", draw, angle radius=.5cm, \cint] {angle=C--B--A}; 
    \pic [\thA, "$\beta_{\ki}$", draw, <-, angle radius=.7cm, angle eccentricity=1.4, \cobs] {angle=C--B--X}; 
    \pic [\thA, "$\beta_{\ki-1}$", draw, <-, angle radius=1cm, angle eccentricity=1.4, \cobs] {angle=B--A--X}; 
    \pic [\thA, "$\delta_\ki$", draw, angle radius=.7cm, angle eccentricity=1.25, \cgam] {angle=X--B--A}; 
    \pic [\thA, "$\alpha_\ki$", draw, ->, angle radius=1cm, angle eccentricity=.65, \csub] {angle=A--X--B}; 

    \fill (A) circle (\dsz) node[left] {$v_{\ki-1}$};
    \fill (B) circle (\dsz) node[left] {$v_{\ki}$};
    \fill (X) circle (\dsz) node[below] {$x$};

\end{tikzpicture} \\
         \huge{$+$} & \huge{$-$} & \tikzsetnextfilename{pmConvex}
\begin{tikzpicture}[scale=1]
    \coordinate (X) at (1.5,-1.3);
    \coordinate (A) at (4.5,1.4);
    \coordinate (B) at (2,1);
    \coordinate (C) at (3.5,-1);
    \coordinate (D) at (4,-1);
    \draw[fill_P_side_light] (A) -- (B) -- (C) [rounded corners]-- (D) [sharp corners]-- cycle;
    \draw[\thL, postaction={decorate,decoration={markings,mark=at position 0.5 with {\arrow{>}}}}] 
        (A) -- (B) node[pos=.5, above] {$e_{k}$};
    \draw[\thL] (B) -- (C) node[pos=.9, left] {$e_{k+1}$};
    \draw[\thL, \cobs] (X) -- (B);
    \draw[\thL, \cobs] (X) -- (A);
    \pic [\thA, "$\eta_\ki$", draw, angle radius=.7cm, \cint] {angle=C--B--A}; 
    \pic [\thA, "$\beta_{\ki}$", draw, ->, angle radius=.7cm, \cobs] {angle=X--B--C}; 
    \pic [\thA, "$\beta_{\ki-1}$", draw, <-, angle radius=.8cm, angle eccentricity=1.4, \cobs] {angle=B--A--X}; 
    \pic [\thA, "$\delta_\ki$", draw, angle radius=.9cm, angle eccentricity=1.25, \cgam] {angle=X--B--A}; 
    \pic [\thA, "$\alpha_\ki$", draw, ->, angle radius=1.1cm, angle eccentricity=.65, \csub] {angle=A--X--B}; 

    \fill (A) circle (\dsz) node[right] {$v_{\ki-1}$};
    \fill (B) circle (\dsz) node[above] {$v_{\ki}$};
    \fill (X) circle (\dsz) node[right] {$x$};

\end{tikzpicture} & \tikzsetnextfilename{pmConcave}
\begin{tikzpicture}[scale=1]
    \coordinate (X) at (5,2);
    \coordinate (A) at (2,2);
    \coordinate (B) at (4,0);
    \coordinate (C) at (3.2,-1);
    \coordinate (D) at (4.5,-1);
    \coordinate (E) at (5,0);
    \coordinate (F) at (4.5,2);
    \draw[fill_P_side_light] (A) -- (B) -- (C) [rounded corners]-- (D) -- (E) -- (F) [sharp corners]-- cycle;
    \draw[\thL, postaction={decorate,decoration={markings,mark=at position 0.6 with {\arrow{>}}}}] 
        (A) -- (B) node[pos=.4, below left, inner sep=1pt] {$e_{k}$};
    \draw[\thL] (B) -- (C) node[pos=.8, left] {$e_{k+1}$};
    \draw[\thL, \cobs] (X) -- (B);
    \draw[\thL, \cobs] (X) -- (A);
    \pic [\thA, "$\eta_\ki$", draw, angle radius=.65cm, \cint] {angle=C--B--A}; 
    \pic [\thA, "$\beta_{\ki}$", draw, ->, angle radius=.5cm, angle eccentricity=1.6, \cobs] {angle=X--B--C}; 
    \pic [\thA, "$\beta_{\ki-1}$", draw, <-, angle radius=1.4cm, angle eccentricity=.7, \cobs] {angle=B--A--X}; 
    \pic [\thA, "$\delta_\ki$", draw, angle radius=.8cm, angle eccentricity=1.25, \cgam] {angle=X--B--A}; 
    \pic [\thA, "$\alpha_\ki$", draw, ->, angle radius=1cm, \csub] {angle=A--X--B}; 

    \fill (A) circle (\dsz) node[below, xshift=-1mm, yshift=-.3mm] {$v_{\ki-1}$};
    \fill (B) circle (\dsz) node[left] {$v_{\ki}$};
    \fill (X) circle (\dsz) node[below] {$x$};

    \node at (2.5,0) {\huge \ding{55}};   %
\end{tikzpicture} \\
         \huge{$-$} & \huge{$+$} & \tikzsetnextfilename{mpConvex}
\begin{tikzpicture}[scale=1]
    \coordinate (X) at (1.2,1.3);
    \coordinate (A) at (3.6,1.3);
    \coordinate (B) at (2,-1);
    \coordinate (C) at (3.2,-1.6);
    \coordinate (D) at (4,-1.6);
    \draw[fill_P_side_light] (A) -- (B) -- (C) [rounded corners]-- (D) [sharp corners]-- cycle;
    \draw[\thL, postaction={decorate,decoration={markings,mark=at position 0.5 with {\arrow{>}}}}] 
        (A) -- (B) node[pos=.5, left] {$e_{k}$};
    \draw[\thL] (B) -- (C) node[pos=.9, left] {$e_{k+1}$};
    \draw[\thL, \cobs] (X) -- (B);
    \draw[\thL, \cobs] (X) -- (A);
    \pic [\thA, "$\eta_\ki$", draw, angle radius=.8cm, \cint] {angle=C--B--A}; 
    \pic [\thA, "$\beta_{\ki}$", draw, <-, angle radius=1cm, angle eccentricity=1.25, \cobs] {angle=C--B--X}; 
    \pic [\thA, "$\beta_{\ki-1}$", draw, ->, angle radius=.9cm, angle eccentricity=1.4, \cobs] {angle=X--A--B}; 
    \pic [\thA, "$\delta_\ki$", draw, angle radius=.8cm, \cgam] {angle=A--B--X}; 
    \pic [\thA, "$\alpha_\ki$", draw, <-, angle radius=.9cm, \csub] {angle=B--X--A}; 

    \fill (A) circle (\dsz) node[right] {$v_{\ki-1}$};
    \fill (B) circle (\dsz) node[below] {$v_{\ki}$};
    \fill (X) circle (\dsz) node[left] {$x$};

    \node at (1.4,-.7) {\huge \ding{55}};   %
\end{tikzpicture} & \tikzsetnextfilename{mpConcave}
\begin{tikzpicture}[scale=1]
    \coordinate (X) at (3.5,-2);
    \coordinate (A) at (4,1);
    \coordinate (B) at (5.5,-0.5);
    \coordinate (C) at (3.3,-.85);
    \coordinate (C2) at (3.3,-1.5);
    \coordinate (D) at (6,-1.5);
    \coordinate (E) at (6,0);
    \coordinate (F) at (5.5,1);
    \draw[fill_P_side_light] (A) -- (B) -- (C) [rounded corners]-- (C2) -- (D) -- (E) -- (F) [sharp corners]-- cycle;
    \draw[\thL, postaction={decorate,decoration={markings,mark=at position 0.5 with {\arrow{>}}}}] 
        (A) -- (B) node[pos=.5, right] {$e_{k}$};
    \draw[\thL] (B) -- (C) node[pos=.9, above] {$e_{k+1}$};
    \draw[\thL, \cobs] (X) -- (B);
    \draw[\thL, \cobs] (X) -- (A);
    \pic [\thA, "$\eta_\ki$", draw, angle radius=.5cm, \cint] {angle=C--B--A}; 
    \pic [\thA, "$\beta_{\ki}$", draw, <-, angle radius=.65cm, angle eccentricity=1.6, \cobs] {angle=C--B--X}; 
    \pic [\thA, "$\beta_{\ki-1}$", draw, ->, angle radius=.6cm, angle eccentricity=1.5, \cobs] {angle=X--A--B}; 
    \pic [\thA, "$\delta_\ki$", draw, angle radius=.8cm, angle eccentricity=1.25, \cgam] {angle=A--B--X}; 
    \pic [\thA, "$\alpha_\ki$", draw, <-, angle radius=.8cm, angle eccentricity=.65, \csub] {angle=B--X--A}; 

    \fill (A) circle (\dsz) node[left] {$v_{\ki-1}$};
    \fill (B) circle (\dsz) node[left] {$v_{\ki}$};
    \fill (X) circle (\dsz) node[left] {$x$};

\end{tikzpicture}\\
         \huge{$-$} & \huge{$-$} & \tikzsetnextfilename{mmConvex}
\begin{tikzpicture}[scale=1]
    \coordinate (X) at (1.8,0);
    \coordinate (A) at (4,2);
    \coordinate (B) at (4,0);
    \coordinate (C) at (4.75,-1);
    \coordinate (D) at (5.3,-1);
    \coordinate (E) at (5.3,2);
    \draw[fill_P_side_light] (A) -- (B) -- (C) [rounded corners]-- (D) -- (E) [sharp corners]-- cycle;
    \draw[\thL, postaction={decorate,decoration={markings,mark=at position 0.57 with {\arrow{>}}}}] 
        (A) -- (B) node[pos=.5, right] {$e_{k}$};
    \draw[\thL] (B) -- (C) node[pos=.8, right] {$e_{k+1}$};
    \draw[\thL, \cobs] (X) -- (B);
    \draw[\thL, \cobs] (X) -- (A);
    \pic [\thA, "$\eta_\ki$", draw, angle radius=.8cm, \cint] {angle=C--B--A}; %
    \pic [\thA, "$\beta_{\ki}$", draw, ->, angle radius=.8cm, \cobs] {angle=X--B--C}; 
    \pic [\thA, "$\beta_{\ki-1}$", draw, ->, angle radius=.8cm, angle eccentricity=1.35, xshift=-10, \cobs] {angle=X--A--B}; 
    \pic [\thA, "$\delta_\ki$", draw, angle radius=.8cm, \cgam] {angle=A--B--X}; 
    \pic [\thA, "$\alpha_\ki$", draw, <-, angle radius=1.1cm, angle eccentricity=.65, \csub] {angle=B--X--A}; 

    \fill (A) circle (\dsz) node[right] {$v_{\ki-1}$};
    \fill (B) circle (\dsz) node[below] {$v_{\ki}$};
    \fill (X) circle (\dsz) node[below] {$x$};

\end{tikzpicture} & \tikzsetnextfilename{mmConcave}
\begin{tikzpicture}[scale=1]
    \coordinate (X) at (0,0.5);
    \coordinate (A) at (1.5,2);
    \coordinate (B) at (3,0);
    \coordinate (C) at (2,-1);
    \coordinate (D) at (3.5,-1);
    \coordinate (E) at (4,0);
    \coordinate (F) at (3.5,2);
    \draw[fill_P_side_light] (A) -- (B) -- (C) [rounded corners]-- (D) -- (E) -- (F) [sharp corners]-- cycle;
    \draw[\thL, postaction={decorate,decoration={markings,mark=at position 0.5 with {\arrow{>}}}}] 
        (A) -- (B) node[pos=.5, right] {$e_{k}$};
    \draw[\thL] (B) -- (C) node[pos=.85, left] {$e_{k+1}$};
    \draw[\thL, \cobs] (X) -- (B);
    \draw[\thL, \cobs] (X) -- (A);
    \pic [\thA, "$\eta_\ki$", draw, angle radius=.7cm, \cint] {angle=C--B--A}; %
    \pic [\thA, "$\beta_{\ki}$", draw, ->, angle radius=.7cm, angle eccentricity=1.4, \cobs] {angle=X--B--C}; 
    \pic [\thA, "$\beta_{\ki-1}$", draw, ->, angle radius=1cm, \cobs] {angle=X--A--B}; 
    \pic [\thA, "$\delta_\ki$", draw, angle radius=.7cm, angle eccentricity=1.4, \cgam] {angle=A--B--X}; 
    \pic [\thA, "$\alpha_\ki$", draw, <-, angle radius=1cm, \csub] {angle=B--X--A}; 

    \fill (A) circle (\dsz) node[left] {$v_{\ki-1}$};
    \fill (B) circle (\dsz) node[below,xshift=2pt,yshift=2pt] {$v_{\ki}$};
    \fill (X) circle (\dsz) node[below] {$x$};

\end{tikzpicture}
    \end{tabular}

    \caption{Angles that are relevant for the vertex-edge pair $(v_k,e_{k+1})$, depending on the relative position of $x$ with respect to the edges incident with $v_k$ and on the convexity of the corner. Due to the choice of the cyclic order on $V(P)$, the polygon (shown in grey) is always on the left hand side when traversing $e_k$ from $v_{k-1}$ to $v_k$. Cases marked with \ding{55} are the ones that occur in \eqref{eq:inside_outside_bounded_simplified}.}
    \label{fig:angle_relations}
\end{figure}
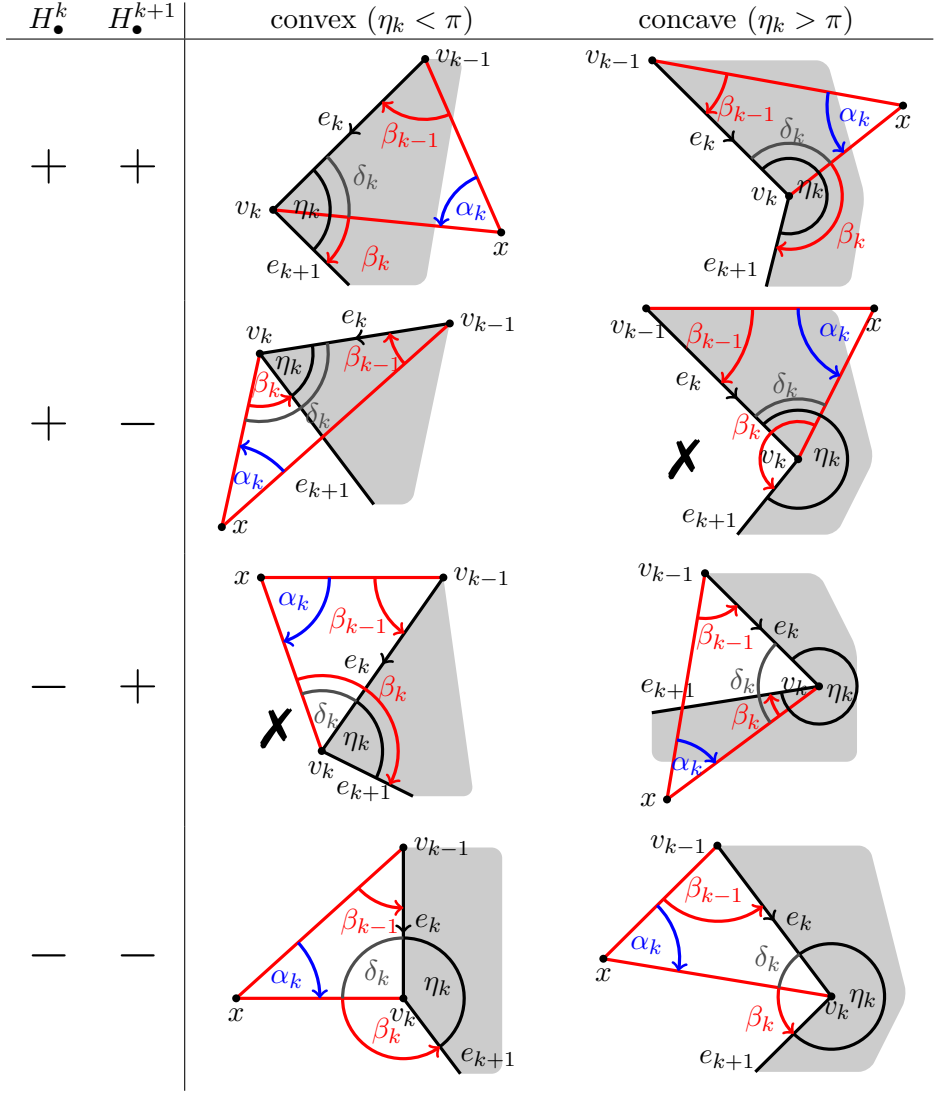

Now, consider the case $P=\cl{\exterior(\gamma)}$, i.e. $\gamma$ is a hole and $n_h(P)=1$.
Then, $P^c:=\overline{\RR^2\setminus P}=\cl{\interior(\gamma)}$ is a polygon with $\partial P^c = \partial P$, i.e. $V(P^c)=V(P)$ and $E(P^c)=E(P)$. 
By the definition of $P$-sides, we have $Q_{P^c}^v=\overline{\RR^2\setminus Q^v_P}$ and $H_{P^c}^e=\overline{\RR^2\setminus H^e_P}$.
Since $x$ is assumed to be in $P$-general position, this implies $\I{Q^v_P}(x)=1-\I{Q_{P^c}^v}(x)$ and $\ones_{H^e_P}(x)=1-\ones_{H_{P^c}^e}(x)$, as well as $\I{P}(x)=1-\I{P^c}(x)$.
Therefore,
\begin{align*}
    v_P(x)-e_P(x) &= \sum_{v\in V(P)} \ones_{Q^v_P}(x)-\sum_{e\in E_b(P)} \ones_{H^e_P}(x) \\
    &= \sum_{v\in V(P)} (1-\ones_{Q_{P^c}^v}(x))-\sum_{e\in E_b(P)} (1-\ones_{H_{P^c}^e}(x)) \\
    &= -\Big(\sum_{v\in V(P)} \I{Q_{P^c}^v}(x) - \sum_{e\in E_b(P)} \I{H_{P^c}^e}(x) \Big)
\end{align*}
where we have used that $|V(P)|=|E(P)|$ in the last step.
Since $P^c=\cl{\interior(\gamma)}$, we can apply the previous case to $P^c$, and get
\begin{align*}
    v_P(x)-e_P(x)-n_h(P) = 1-\I{P^c}(x)-1 = \I{P}(x)-1.
\end{align*}
This proves the remaining case of \eqref{eq:inside_outside_lem_bd}.
\end{proof}

\FloatBarrier
\subsubsection{Only arcs} \label{sec:inside_outside_arcs}
In this section, I prove the following special case of \autoref{lem:InsideOutside}.
\begin{lemma}
\label{lem:InsideOutsideArcs}
    Let $P$ be a polygon whose boundary consists of $n_a(P)\geq 1$ polygonal arcs. Let $x\in\RR^2$ be an arbitrary point in $P$-general position.
    Then, it holds that
    \begin{multline}\label{eq:inside_outside_arcs}
        \sum_{v\in V(P)}\I{Q^v_P}(x)+\sum_{e\in E_l(P)}\I{H^e_P}(x)-\sum_{e\in \Eb(P)}\I{H^e_P}(x)-n_a(P)
        \\
         = \I{P}(x)-1.
    \end{multline}
\end{lemma}

The boundary of any polygon $P$ that satisfies the assumptions of the lemma contains no cycles and consequently no holes, i.e. $n_h(P)=0$. Moreover, as we will see in \autoref{lem:arcInterstions}, we again have $d(P)=0$. This gives $c(P)=1-n_a(P)$ and confirms that \autoref{lem:InsideOutsideArcs} is a special case of \autoref{lem:InsideOutside}.

As the next lemma shows, it is sufficient to restrict attention to polygons that contain no lines as boundary components.
\begin{definition}
	Let $P$ be a polygon, and let $l\in E_l(P)$ be a line. 
	Choose two rays $e_1$ and $e_2$ such that $e_1\cup e_2=l$ and $e_1\cap e_2 = \set{v}$ for some $v\in \RR^2$.
	The operation of redefining $V(l):=\set{v}$ and $E(l):=\set{e_1,e_2}$, and updating the vertex and edge sets of $P$ accordingly, is called \emph{splitting} the line $l$ at $v$.
\end{definition}
\begin{lemma}\label{lem:splitLine}
	For any polygon $P$, both sides of \eqref{eq:inside_outside_arcs} are invariant under the operation of splitting a line.
\end{lemma}
\begin{proof}
	Let $l\in E_l(P)$ be a line that is split at $v\in l$.
	The right hand side of \eqref{eq:inside_outside_arcs} is invariant under the operation of splitting a line because it does not depend on the specific choice of vertices and edges.
	Additionally, the number of arcs $n_a(P)$ and the set of line segments $\Eb(P)$ remain unchanged by this operation.
	The $P$-side of $v$ is a half-plane that agrees with the $P$-side of $l$, see \autoref{fig:splitLine}. Therefore, $\I{Q_P^v}=\I{H_P^l}$, and since the operation adds $v$ to $V(P)$ and removes $l$ from $E_l(P)$, it does not affect the left hand side.
\end{proof}

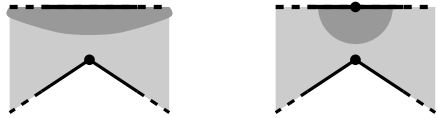
\begin{figure}[h!]
	\centering
	\def\rlab{.1}   %

\tikzsetnextfilename{splitLine}
\begin{tikzpicture}[very thick, scale = .7]

	\coordinate (lL) at (-1.5,1);
	\coordinate (lR) at (1.5,1);
	\coordinate (lM) at (0,1);
	\coordinate (M) at (0,0);
	\coordinate (bL) at (-1.5,-1);
	\coordinate (bR) at (1.5,-1);

	\draw[fill_P_side_light] (lL) -- (lR) -- (bR) -- (M) -- (bL) -- cycle;
	\draw[fill_P_side] plot [smooth] coordinates 
	{ (lL) (-1.5,.8) (-.5,.5) (.5,.5) (1.5,.8) (lR) } -- cycle;
	
	\draw[ultra thick, dashed] (lL)-- ($(lL)!.2!(lR)$);
	\draw[ultra thick] ($(lL)!.2!(lR)$) -- ($(lL)!.8!(lR)$);
	\draw[ultra thick, dashed] ($(lL)!.2!(lR)$) -- (lR);
	
	\draw[] (M) -- ($(M)!.6!(bL)$);
	\draw[dashed] ($(M)!.6!(bL)$) -- (bL) ;
	\draw[] (M) -- ($(M)!.6!(bR)$);
	\draw[dashed] ($(M)!.6!(bR)$) -- (bR) ;
	
	\fill (M) circle (\rlab);

	\begin{scope}[shift={(5,0)}]	
		\coordinate (lL) at (-1.5,1);
		\coordinate (lR) at (1.5,1);
		\coordinate (lM) at (0,1);
		\coordinate (M) at (0,0);
		\coordinate (bL) at (-1.5,-1);
		\coordinate (bR) at (1.5,-1);
	
		\draw[fill_P_side_light] (lL) -- (lR) -- (bR) -- (M) -- (bL) -- cycle;
		\pic [fill_P_side, angle radius=14pt] {angle=lL--lM--lR}; 
		
		\draw[ultra thick, dashed] (lL)-- ($(lL)!.2!(lR)$);
		\draw[ultra thick] ($(lL)!.2!(lR)$) -- ($(lL)!.8!(lR)$);
		\draw[ultra thick, dashed] ($(lL)!.2!(lR)$) -- (lR);
		
		\draw[] (M) -- ($(M)!.6!(bL)$);
		\draw[dashed] ($(M)!.6!(bL)$) -- (bL) ;
		\draw[] (M) -- ($(M)!.6!(bR)$);
		\draw[dashed] ($(M)!.6!(bR)$) -- (bR) ;
		
		\fill (M) circle (\rlab);
		\fill (lM) circle (\rlab);
	\end{scope}
	
\end{tikzpicture}
	\caption{A line of a polygon is split by a vertex. The propagation of the $P$-sides of the line and vertex are indicated in dark grey.}
	\label{fig:splitLine}
\end{figure}

The idea for the proof of \autoref{lem:InsideOutsideArcs} is to carefully modify the polygon by adding auxiliary edges and vertices in such a way that its boundary becomes a polygonal cycle to which \autoref{lem:InsideOutsideBounded} can be applied. 
To do so, we will cut off two rays and connect the remainder of the corresponding arcs with line segments.

The following lemma will ensure that the modified boundary components remain simple curves.
\begin{lemma}%
\label{lem:arcInterstions}
    Let $P$ be a polygon. For each vertex $v\in V(P)$, there is at most one polygonal arc containing $v$ that is a boundary component of $P$.
\end{lemma}
\begin{proof}
    Assume that $\gamma_1$ and $\gamma_2$ are two polygonal arcs that intersect at a vertex $v\in V(P)$.
    Let $D$ be a disk large enough such that $V(P)\subset D$.
    Due to the convexity of the disk, all line segments are contained in $D$.
    If $e\in E(P)$ is a ray, $e$ intersects $\partial D$, since its vertex is contained in $D$ and $e$ is unbounded. Moreover, there is exactly one such intersection because $e$ is straight.
    Therefore, $\gamma_1$ and $\gamma_2$ create four intersection points of $\partial P$ with $\partial D$, and subdivide $\RR^2\setminus D$ into four sectors (see \autoref{fig:intersecting_intersections}).
    Let $y_1,y_2\in P$ be points that are contained in different sectors $S_1$ and $S_2$. Then, from the connectedness of $\interior P$, it follows that there is a path $\lambda\subset\interior P$ connecting $y_1$ with $y_2$. $\lambda$ is not allowed to intersect $\gamma_i\subset\partial P$, i.e., it must pass through $D$.
    Therefore, there are two points $x_1,x_2\in\partial D$, with $x_2\in S_2$ and $x_1\not\in S_2$, that are connected by a segment $\lambda'\subset\lambda$ completely contained within $D$.
    If $r_3,r_4$ are the rays that define the sector $S_2$, let $x_3,x_4$ be their respective intersections with $\partial D$. Then, since $v\in\gamma_1\cap\gamma_2$, there is also a path $\gamma'\subset \cl{D}\cap \partial P$ with endpoints $x_3$ and $x_4$.
    By compactness of $\lambda$, we can assume that $\lambda'$ consists of finitely many line segments, i.e. $\lambda'\cup\gamma'$ forms a planar graph, embedded in $D$.
    Since the intersection points occur in the order $x_1,x_3,x_2,x_4$ cyclically on $\partial D$, $\lambda'$ and $\gamma'$ must intersect, see e.g. \citet[Ch. 71]{schrijver2003}. 
    Thus, there is a point $x\in\lambda$ that lies on the boundary of $P$, which is a contradiction.
\end{proof}

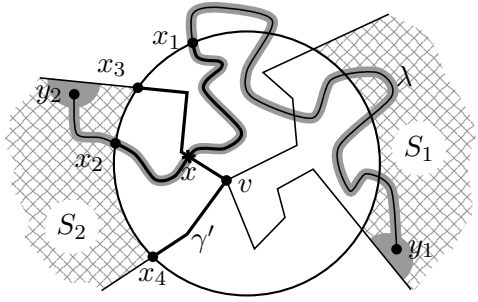
\begin{figure}
    \centering
    \def\rBR{2.5}     %
\def\rlab{.1}   %

\newcommand{\arcDleft}{ (d) -- (xd) -- (170:.5*\rBR) -- (v) }
\newcommand{\arcDright}{ (v) -- (20:.4*\rBR) -- (55:.6*\rBR) -- (xa) -- (a) }
\newcommand{\arcD}{ \arcDleft -- \arcDright }
\newcommand{\arcBright}{ (b) -- (xb) -- (-40:.3*\rBR) -- (-55:.5*\rBR) -- (275:.65*\rBR) -- (v) }
\newcommand{\arcBleft}{ (v) -- (xc) -- (c) }
\newcommand{\arcB}{ \arcBright -- \arcBleft }

\newcommand{\plotConnPath}{plot [smooth] coordinates 
	{ (y1) (-10:1.1*\rBR) (-5:.9*\rBR) (-15:.8*\rBR) (10:.7*\rBR) (25:1.2*\rBR) (35:1.2*\rBR) (45:.9*\rBR) (35:.6*\rBR) (85:.5*\rBR)
		(90:.7*\rBR) (80:1.1*\rBR) (100:1.2*\rBR) (110:1.2*\rBR) (115:.9*\rBR) (105:.8*\rBR) (105:.75*\rBR) (120:.65*\rBR)
		(100:.25*\rBR) (145:.2*\rBR) (165:.4*\rBR) (190:.55*\rBR) (190:.7*\rBR) (175:.85*\rBR) (170:1.1*\rBR) (170:1.3*\rBR)
		(y2) }}

\tikzsetnextfilename{arcsIntersection}
\begin{tikzpicture}[very thick,
       tangent/.style = {gray, thick, dashed},
       scale = .7
    ]
    \tikzset{cross/.style={cross out, draw=black, minimum size=2*(#1-\pgflinewidth),        inner sep=0pt, outer sep=0pt},
        cross/.default={3pt}}

    \coordinate (O) at (0,0);
    \coordinate (v) at (220:.2*\rBR);
    \coordinate (a) at (60:\rBR);
    \coordinate (b) at (-30:\rBR);
    \coordinate (c) at (225:\rBR);
    \coordinate (d) at (145:\rBR);
    \coordinate (y1) at (-30:1.3*\rBR);
    \coordinate (y2) at (158:1.4*\rBR);
    \coordinate (xa) at (80:.7*\rBR);
    \coordinate (xb) at (-5:.5*\rBR);
    \coordinate (xc) at (230:.7*\rBR);
    \coordinate (xd) at (130:.7*\rBR);
    \coordinate (ea) at ($(xa)!2.5!(a)$);
    \coordinate (eb) at ($(xb)!2!(b)$);
    \coordinate (ec) at ($(xc)!2.5!(c)$);
    \coordinate (ed) at ($(xd)!3.0!(d)$);

    \draw[thick, name path=circle] (O) circle (\rBR);

    \draw[name path=gammaAleft] \arcDleft;
    \draw[postaction={decorate,decoration={markings,mark=at position 0.7 with {\node[right] {$\gamma'$};}}}] 
    	\arcBleft;
    \draw[semithick, name path=gammaAright] \arcDright;
    \draw[semithick] \arcBright;
    \draw[semithick] (a) -- (ea);
    \draw[semithick] (b) -- (eb);
    \draw[semithick] (c) -- (ec);
    \draw[semithick] (d) -- (ed);

    \draw[draw=none, name path=connPath, 
    	postaction={decorate,decoration={markings,mark=at position 0.26 with {\node[right] {$\lambda$};}}}] 
    	\plotConnPath;
        
    \begin{scope}[on background layer]
    	\draw[Pcolor, line width=4pt] \plotConnPath;
    \end{scope}

	\path[draw, intersection segments={of=connPath and circle,sequence={L6}}];
    \path[draw, semithick, intersection segments={of=connPath and circle,sequence={L1}}];
    \path[draw, semithick, intersection segments={of=connPath and circle,sequence={L2}}];
    \path[draw, semithick, intersection segments={of=connPath and circle,sequence={L3}}];
    \path[draw, semithick, intersection segments={of=connPath and circle,sequence={L4}}];
    \path[draw, semithick, intersection segments={of=connPath and circle,sequence={L5}}];
    \path[draw, semithick, intersection segments={of=connPath and circle,sequence={L7}}];

    \begin{scope}[on background layer]

        \coordinate (y21) at ($(xd)$);
        \coordinate (y22) at ($(xd)!2.3!(d)$);
        \coordinate (y23) at ($(xd)!3!(d)$);
        \pic [fill_P_side, angle radius=10pt] {angle=y23--y22--y21};

		\coordinate (y11) at ($(xb)$);
		\coordinate (y12) at ($(xb)!1.5!(b)$);
		\coordinate (y13) at ($(xb)!3!(b)$);
		\pic [fill_P_side, angle radius=10pt] {angle=y13--y12--y11}; 
        
        \draw[pattern color=Pcolor,
        	rounded corners=2pt, draw=none,
        	pattern={Hatch[angle=45,distance={5pt/sqrt(2)},xshift=.1pt]},
        	decoration={random steps,segment length=7pt,amplitude=3pt}] 
        	(a) -- (ea) decorate { to [bend left=60] (eb) } -- (b) arc (-30:60:\rBR);
        \draw[pattern color=Pcolor,
        	rounded corners=2pt, draw=none,
        	pattern={Hatch[angle=45,distance={5pt/sqrt(2)},xshift=.1pt]},
        	decoration={random steps,segment length=7pt,amplitude=3pt}] 
        	(c) -- (ec) decorate { to [bend left=60] (ed) } -- (d) arc (145:225:\rBR);
    \end{scope}

    \fill (v) circle (\rlab) node[right, yshift=-2pt] {$v$};
    
    \fill (c) circle (\rlab) node[below] {$x_4$};
    \fill (d) circle (\rlab) node[above left] {$x_3$};

    \path [name intersections={of=connPath and circle, name=inter}];
    \fill (inter-5) circle (\rlab) node[left, yshift=1pt] {$x_1$};
    \fill (inter-6) circle (\rlab) node[below left] {$x_2$};

    \path [name intersections={of=connPath and gammaAleft, name=i, total=\t}]
        \foreach \s in {1,...,\t}{(i-\s) node[cross,rotate=55] (cross\s) {}};
    \node[below] at (cross1) {$x$};
    
    \fill (y1) circle (\rlab) node[right] {$y_1$};
    \fill (y2) circle (\rlab) node[left] {$y_2$};
    
    \node[fill=white, circle, inner sep=2pt] at (5:1.3*\rBR) {$S_1$};
    \node[fill=white, circle, inner sep=2pt] at (200:1.4*\rBR) {$S_2$};
    
\end{tikzpicture}
    \caption{Two arcs subdivide $\RR^2\setminus D$ into four sectors. $y_1,y_2\in P$ are contained in different sectors and are connected by a path $\lambda$ which enters and leaves each sector only through $\partial D$.
    	If the arcs share a vertex in $D$, then $\lambda$ must intersect one of the arcs.}
    \label{fig:intersecting_intersections}
\end{figure}

Next, I will constructively show that, for fixed $x$, the number of polygonal arcs in $\partial P$ can be reduced without affecting \eqref{eq:inside_outside_arcs}.

\begin{lemma}\label{lem:close_end}
    Let $P$ be a polygon whose boundary consists of $n_a(P)\geq 1$ polygonal arcs without lines, and let $x\in\RR^2$ be in $P$-general position. 
    Then, there is a polygon $P'$, also without lines and with $x$ in $P'$-general position, such that $\partial P'$ consists of $n_a(P')=n_a(P)-1$ arcs and both sides of \eqref{eq:inside_outside_arcs} are preserved at $x$. That is, \begin{equation}\label{eq:lem_merge_arcs_rhs}
    	\I{P}(x)-1=\I{P'}(x)-1, %
    \end{equation}
    and,
    \begin{multline}\label{eq:lem_merge_arcs_lhs}
    	\sum_{v\in V(P)}\I{Q^v_P}(x)-\sum_{e\in \Eb(P)}\I{H^e_P}(x)-n_a(P)\\
    	= \sum_{v\in V(P')}\I{Q^v_{P'}}(x)-\sum_{e\in \Eb(P')}\I{H^e_{P'}}(x)-n_a(P') .%
    \end{multline}
    Moreover, if $n_a(P)=1$, then $P'$ is bounded and $\partial P'$ is a cycle.
\end{lemma}

\begin{proof}
    Let $D$ be a disk, centered at $0$ and large enough such that $\set{x}\cup V(P)\subset D$.

    Convexity of $D$ implies that all line segments are contained in $D$.
    If $e\in E(P)$ is a ray, $e$ intersects $\partial D$, since its vertex is contained in $D$ and $e$ is unbounded. 
    Moreover, the rays intersect $\partial D$ non-tangentially, so they subdivide $\partial D$ into circular arcs whose interior is alternately contained in $P$ and in its complement.

	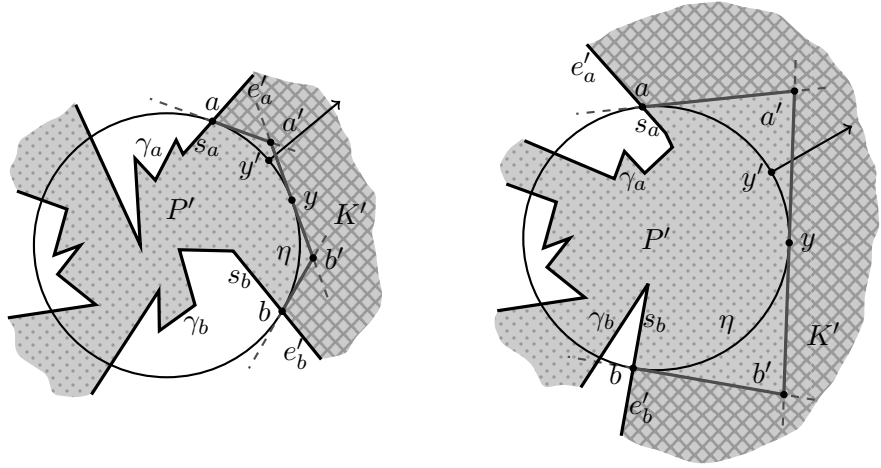
\begin{figure}
		\centering
		\begin{tabular}{C{5.7cm} C{6cm}}
			\def\rBR{2.5}     %
\def\rlab{.07}   %

\definecolor{closingColor}{gray}{0.3}

\newcommand{\arcF}{ (f) -- (xf) -- (110:.7*\rBR) -- (100:.5*\rBR) -- (85:.8*\rBR) -- (xa) -- (a) }
\newcommand{\arcB}{ (b) -- (xb) -- (-20:.1*\rBR) -- (-65:.5*\rBR) -- (265:.65*\rBR) -- (xc) -- (c) }
\newcommand{\arcD}{ (d) -- (xd) -- (195:.85*\rBR) -- (180:.65*\rBR) -- (185:.85*\rBR)  -- (xe) -- (e) }

\tikzsetnextfilename{closeEndSmall}
\begin{tikzpicture}[very thick,
       tangent/.style = {closingColor, thick, dashed},
       scale = .7
    ]
    \coordinate (O) at (0,0);
    \coordinate (a) at (70:\rBR);
    \coordinate (b) at (-30:\rBR);
    \coordinate (y) at (20:\rBR);
    \coordinate (y') at (40:\rBR);
    \coordinate (c) at (245:\rBR);
    \coordinate (d) at (210:\rBR);
    \coordinate (e) at (160:\rBR);
    \coordinate (f) at (125:\rBR);
    \coordinate (xa) at (80:.7*\rBR);
    \coordinate (xb) at (-5:.5*\rBR);
    \coordinate (xc) at (260:.35*\rBR);
    \coordinate (xd) at (220:.7*\rBR);
    \coordinate (xe) at (160:.8*\rBR);
    \coordinate (xf) at (180:.2*\rBR);
    \coordinate (ea) at ($(xa)!2.4!(a)$);
    \coordinate (eb) at ($(xb)!1.8!(b)$);
    \coordinate (ec) at ($(xc)!1.4!(c)$);
    \coordinate (ed) at ($(xd)!2!(d)$);
    \coordinate (ee) at ($(xe)!2!(e)$);
    \coordinate (ef) at ($(xf)!1.3!(f)$);

    \draw[thick] (O) circle (\rBR);

    \draw[->, thick] (y') -- ($(O)!1.7!(y')$);

    \draw[tangent,name path=tk]     (a) -- ($(a)!.7!90:(O)$);  %
    \draw[tangent]                  (a) -- ($(a)!.5!-90:(O)$);   %
    \draw[tangent]                  (b) -- ($(b)!.5!90:(O)$);  
    \draw[tangent,name path=tkp]    (b) -- ($(b)!.7!-90:(O)$);  
    \draw[tangent,name path=tk0l]   (y) -- ($(y)!.8!-90:(O)$);  
    \draw[tangent,name path=tk0r]   (y) -- ($(y)!.8!90:(O)$); 

    \draw[closingColor, 
            name intersections={of=tk and tk0l, by=a0},
            name intersections={of=tkp and tk0r, by=b0}] 
        (a) -- (a0) -- (b0) -- (b);

    \draw (ef) -- \arcF -- (ea);
    \draw (eb) -- \arcB -- (ec);
    \draw (ed) -- \arcD -- (ee);

    \begin{scope}[on background layer]
        \draw[preaction={fill, PcolorLight}, pattern color=Pcolor,
            rounded corners=2pt, draw=none, line width=1pt,
            pattern={Hatch[angle=45,distance={6pt/sqrt(2)},xshift=.1pt]},
            decoration={random steps,segment length=7pt,amplitude=3pt}]  
            (ea) -- (a) -- (a0) -- (b0) -- (b) -- (eb) decorate { to [bend right=90] (ea) };
        
        \draw[preaction={fill, PcolorLight}, pattern color=Pcolor,
            rounded corners=2pt, draw=none,
            pattern={Dots[angle=45,distance={5pt/sqrt(2)},radius=.6pt]},
            decoration={random steps,segment length=7pt,amplitude=3pt}]  
            (ef) -- \arcF -- (a0) -- (b0) -- \arcB -- (ec) decorate { to [bend left=40] (ed) } -- \arcD -- (ee) decorate { to [bend left=40] (ef) };
    \end{scope}

    \fill (a) circle (\rlab) node[above] {$a$};
    \fill (b) circle (\rlab) node[left] {$b$};
    \fill (y) circle (\rlab) node[right] {$y$};
    
    \fill[] (a0) circle (\rlab) node[above right] {$a'$};   %
    \fill[] (b0) circle (\rlab) node[right] {$b'$};         %

    \node at ($(a)-(.1,.5)$) {$s_a$};
    \node at ($(b)+(-.8,.6)$) {$s_b$};
    
    \node at ($(a)+(.9,.6)$) {$e_a'$};
    \node at ($(b)+(.25,-.8)$) {$e_b'$};

    \node at (100:.75*\rBR) {$\gamma_a$};
    \node at (-70:.65*\rBR) {$\gamma_b$};    
    
    \node at ($.8*(b0)$) {$\eta$};

    \fill[] (y') circle (\rlab) node[left, inner sep=2pt, yshift=-2pt] {$y'$};    
    
    \node at (70:.3*\rBR) {$P'$};
    \node at (10:1.4*\rBR) {$K'$};
    
\end{tikzpicture}&
			\def\rBR{2.5}     %
\def\rlab{.07}   %

\definecolor{closingColor}{gray}{0.3}

\newcommand{\arcF}{ (f) -- (xf) -- (110:.7*\rBR) -- (100:.5*\rBR) -- (80:.7*\rBR) -- (xa) -- (a) }
\newcommand{\arcB}{ (b) -- (xb) -- (xc) -- (c) }
\newcommand{\arcD}{ (d) -- (xd) -- (195:.85*\rBR) -- (180:.65*\rBR) -- (185:.85*\rBR)  -- (xe) -- (e) }

\tikzsetnextfilename{closeEndConcave}
\begin{tikzpicture}[very thick, bezier bounding box,
       tangent/.style = {closingColor, thick, dashed},
       scale = .7
    ]
    \coordinate (O) at (0,0);
    \coordinate (a) at (96:\rBR);
    \coordinate (b) at (260:\rBR);
    \coordinate (y) at (-2:\rBR);
    \coordinate (y') at (30:\rBR);
    \coordinate (c) at (245:\rBR);
    \coordinate (d) at (210:\rBR);
    \coordinate (e) at (160:\rBR);
    \coordinate (f) at (140:\rBR);
    \coordinate (xa) at (85:.8*\rBR);
    \coordinate (xb) at (260:.35*\rBR);
    \coordinate (xc) at (xb);
    \coordinate (xd) at (220:.7*\rBR);
    \coordinate (xe) at (160:.8*\rBR);
    \coordinate (xf) at (125:.55*\rBR);
    \coordinate (ea) at ($(xa)!3.4!(a)$);
    \coordinate (eb) at ($(xb)!1.8!(b)$);
    \coordinate (ec) at ($(xc)!1.4!(c)$);
    \coordinate (ed) at ($(xd)!2!(d)$);
    \coordinate (ee) at ($(xe)!2!(e)$);
    \coordinate (ef) at ($(xf)!1.5!(f)$);

    \draw[thick] (O) circle (\rBR);
    
    \draw[->, thick] (y') -- ($(O)!1.7!(y')$);

    \draw[tangent,name path=tk]     (a) -- ($(a)!1.4!90:(O)$);  %
    \draw[tangent]                  (a) -- ($(a)!.5!-90:(O)$);   %
    \draw[tangent]                  (b) -- ($(b)!.5!90:(O)$);  
    \draw[tangent,name path=tkp]    (b) -- ($(b)!1.4!-90:(O)$);  
    \draw[tangent,name path=tk0l]   (y) -- ($(y)!1.4!-90:(O)$);  
    \draw[tangent,name path=tk0r]   (y) -- ($(y)!1.4!90:(O)$); 

    \draw[closingColor, 
            name intersections={of=tk and tk0l, by=a0},
            name intersections={of=tkp and tk0r, by=b0}] 
        (a) -- (a0) -- (b0) -- (b);

    \draw (ef) -- \arcF -- (ea);
    \draw (eb) -- \arcB -- (ec);
    \draw (ed) -- \arcD -- (ee);

    \begin{scope}[on background layer]
        \coordinate (c1) at (-52:3.5*\rBR);
        \coordinate (c2) at (51:4*\rBR);
        \draw[preaction={fill, PcolorLight}, pattern color=Pcolor,
            rounded corners=2pt, draw=none, line width=1pt,
            pattern={Hatch[angle=45,distance={6pt/sqrt(2)},xshift=.1pt]},
            decoration={random steps,segment length=7pt,amplitude=1.5pt}
            ]  
            (b0) -- (b) -- (eb) decorate { .. controls (c1) and (c2) .. (ea) } -- (a) -- (a0) -- cycle ;
            
        \draw[preaction={fill, PcolorLight}, pattern color=Pcolor,
            rounded corners=2pt, draw=none,
            pattern={Dots[angle=45,distance={5pt/sqrt(2)},radius=.6pt]},
            decoration={random steps,segment length=7pt,amplitude=3pt}]  
            (ef) -- \arcF -- (a0) -- (b0) -- \arcB -- (ec) decorate { to [bend left=40] (ed) } -- \arcD -- (ee) decorate { to [bend left=40] (ef) };
    \end{scope}
            
    \fill (a) circle (\rlab) node[above] {$a$};
    \fill (b) circle (\rlab) node[left, yshift=-2pt] {$b$};
    \fill (y) circle (\rlab) node[right] {$y$};
    \fill[] (a0) circle (\rlab) node[below left] {$a'$};   %
    \fill[] (b0) circle (\rlab) node[above left] {$b'$};         %
    
    \node at ($(a)+(.1,-.4)$) {$s_a$};
    \node at ($(b)+(.4,.85)$) {$s_b$};
    
    \node at ($(a)+(-1.1,.75)$) {$e_a'$};
    \node at ($(b)+(.15,-.75)$) {$e_b'$};
    
    \node at (112:.46*\rBR) {$\gamma_a$};
    \node at (-123:.75*\rBR) {$\gamma_b$}; 
    
    \node at ($.55*(b0)$) {$\eta$};
    
    \fill[] (y') circle (\rlab) node[left, inner sep=2pt, yshift=-2pt] {$y'$};    
    
    \node at (O) {$P'$};
    \node at (-30:1.45*\rBR) {$K'$};
    
\end{tikzpicture} 
			\vspace{-1em}
		\end{tabular}
		\caption{The circular arc $\eta\subset P$ between $a$ and $b$ can be approximated by a path $\gamma'$ (dark grey) of three line segments that are contained in tangents. 
		This path can be used to cut off an unbounded area $K'\subset P$ (hatched) from $P$ to form a new polygon $P'$ (dotted).}
		\label{fig:closingEnd}
	\end{figure}

    Let $a,b\in\partial P\cap\partial D$ be the intersection points of two distinct rays $e_a,e_b\in E(P)$ such that the circular arc $\eta$ from $a$ to $b$ in clockwise direction is contained in $P$, cf. \autoref{fig:closingEnd}.
    We subdivide these rays into a line segment $s_a:=e_a\cap \cl{D}$, resp. $s_b:=e_b\cap \cl{D}$, and a ray $e_a':=e_a\setminus D$, resp. $e_b':=e_b\setminus D$.
    Define $a'$ and $b'$ as the intersection points of the tangents to $D$ at $a$ and $b$ with the tangent at the midpoint $y$ of $\eta$ (they do intersect because $y$ cannot be the negative of $a$ and $b$).

    Define $\gamma':=\cl{aa'}\cup\cl{a'b'}\cup\cl{b'b}$, and $K$ as the connected component of $P\setminus D$ between the rays $e_a'$ and $e_b'$.
    Then, due to the choice of $y$ as the midpoint of $\eta$, we have that $\gamma'\subset K$, but $e\cap K=\emptyset$ for all edges $e\in E(P)\setminus\set{e_a,e_b}$ and for $e\in\set{s_a,s_b}$.
    Let $K'$, be the connected component of $\RR^2\setminus (e_a'\cup e_b'\cup \gamma')$ that is contained in $K$, and define $P':=\cl{P\setminus K'}$.
    Then, $P'$ is again a closed set with connected interior. For the connectedness, consider any two points in $P'$. These can be connected by a path in the interior of $P$ that can only enter or leave $K'$ through $\gamma'$. Therefore, the segments of the path that run through $K'$ can be shortened via a path in the interior of $P'$ close to $\gamma'$.
    
    Since $K'\cap D=\emptyset$, we have $P'\cap D=P\cap D$, verifying \eqref{eq:lem_merge_arcs_rhs}.
    
    Let $\gamma_a$ and $\gamma_b$ be the polygonal arcs of $P$ that contain the rays $e_a$ and $e_b$. With
    \begin{equation*}
    	\gamma := (\gamma_a\cup \gamma_b)\setminus(e_a \cup e_b) \cup s_a \cup s_b \cup \gamma',
    \end{equation*}
    we can describe the boundary of $P'$ as \begin{equation} \label{eq:boundary_P'}
        \partial P'=(\partial P\setminus (\gamma_a\cup\gamma_b)) \cup\gamma.
    \end{equation}
    Before addressing \eqref{eq:lem_merge_arcs_lhs}, we prove that $n_a(P')=n_a(P)-1$ and that $P'$ is bounded with its boundary being a cycle if $n_a(P)=1$. 

    \textbf{Case 1:} Assume that $n_a(P)=1$.\\
    In that case, $\gamma_a=\gamma_b$, and $\gamma$ becomes a polygonal cycle.
    Since $\partial P$ consists only of arcs, and $\gamma_a$ is the only arc, we get from \eqref{eq:boundary_P'} that $\partial P'=\gamma$, and $n_a(P')=n_a(P)-1=0$.
    Moreover, for any $y'\in\eta\setminus\set{a,b,y}$, the ray $\set{\rho y':\rho\in[1,\infty)}$ intersects $\gamma$ for exactly one $\rho>0$, since $\gamma\setminus D=\gamma'$. Thus, an initial segment of that ray, and in particular $y'$, is contained in $\interior(\gamma)$. Since $y'\in P'$, we have $P'\subseteq \interior(\gamma)$, thus $P'$ is bounded.
   
    \textbf{Case 2:} Assume that $n_a(P)\geq 2$.\\
    Assume for the sake of contradiction that $\gamma_a=\gamma_b$, i.e. $\gamma_a\cap \partial D =\set{a,b}$.
    Then, as before, $\gamma$ becomes a cycle with $P'\subset \interior(\gamma)$.
    This contradicts the presence of additional polygonal arcs other than $\gamma_a$, as arcs are unbounded by definition.
    Thus, $\gamma_a\neq\gamma_b$, and $\gamma$ is a polygonal arc, due to \autoref{lem:arcInterstions}.
    Together with \eqref{eq:boundary_P'}, this implies $n_a(P')=n_a(P)-1$.
    
    To prove \eqref{eq:lem_merge_arcs_lhs}, note that in either case, $P'$ is a polygon with \begin{equation}\label{eq:P'_expand_V}
        V(P')=V(P)\cup \set{a,b,a',b'},
    \end{equation}
    and \begin{equation}\label{eq:P'_expand_E}
        \Eb(P')=\Eb(P)\cup\set{s_a,\cl{aa'},\cl{a'b'},\cl{b'b},s_b}.
    \end{equation}
    Since $s_a\subset e_a$, $s_b\subset e_b$, and $\cl{aa'},\cl{a'b'},\cl{b'b}$ are tangential to $D$, $x$ is also in $P'$-general position.
    
    The definition of $P$-sides is local in the sense that they are fully determined by any neighbourhood of the vertex or edge under consideration.
    Since $V(P),\Eb(P)\subset D$, the disk $D$ is a neighbourhood of all vertices and line segments. Therefore, $P'\cap D=P\cap D$ implies that 
    \begin{equation}\label{eq:P'_and_P_equal_in_D}
    	Q^v_{P'}=Q^v_{P} \quad\forall v\in V(P)
    	\quad\text{and}\quad 
    	H^e_{P'}=H^e_{P} \quad\forall e\in \Eb(P).
    \end{equation}
    Each line segment $e\in\set{\cl{aa'},\cl{a'b'},\cl{b'b}}$ is tangent to $D$ with $D\subset H^e_{P'}$. Since $x\in D$, we get that 
    \begin{equation}\label{eq:P'_side_tangents}
    	x\in H^e_{P'} \text{ for } e\in \set{\cl{aa'},\cl{a'b'},\cl{b'b}}
    	\quad\text{and}\quad 
    	x\in Q^v_{P'} \text{ for } v\in\set{a',b'}.
    \end{equation}
    This, together with \begin{equation*}
    	Q^a_{P'} = H_{P'}^{s_a}\cap H_{P'}^{\cl{aa'}} 
    	\quad\text{and} \quad
    	Q^b_{P'} = H_{P'}^{s_b}\cap H_{P'}^{\cl{b'b}},
    \end{equation*} 
    yields that
    \begin{equation}\label{eq:P'_side_equivalence}
    	x\in Q^a_{P'} \Leftrightarrow x\in H^{s_a}_{P'}
    	\quad\text{and}\quad 
    	x\in Q^b_{P'} \Leftrightarrow x\in H^{s_b}_{P'}.
    \end{equation}
    Using \eqref{eq:P'_expand_V} and \eqref{eq:P'_expand_E}, we split the sums on the right hand side of \eqref{eq:lem_merge_arcs_lhs}
    \begin{align*}
        \sum_{v\in V(P')}\hspace{-1em}&\hspace{1em}\I{Q^v_{P'}}(x)-\sum_{e\in \Eb(P')}\I{H^e_{P'}}(x)-n_a(P') \\ 
        &= \sum_{v\in V(P)}\I{Q^v_{P'}}(x)+\sum_{v\in \set{a,b}}\I{Q^v_{P'}}(x)+\sum_{v\in \set{a',b'}}\I{Q^v_{P'}}(x)-\sum_{e\in \Eb(P)}\I{H^e_{P'}}(x)\\
        &\quad-\sum_{e\in \set{s_a,s_b}}\I{H^e_{P'}}(x)-\sum_{e\in \set{\cl{aa'},\cl{a'b'},\cl{b'b}}}\I{H^e_{P'}}(x)-n_a(P)+1
    \end{align*}
    Applying first \eqref{eq:P'_and_P_equal_in_D} and \eqref{eq:P'_side_tangents}, and then \eqref{eq:P'_side_equivalence}, gives
    \begin{align*}
    	\sum_{v\in V(P')}\hspace{-1em}&\hspace{1em}\I{Q^v_{P'}}(x)-\sum_{e\in \Eb(P')}\I{H^e_{P'}}(x)-n_a(P') \\ 
        &= \sum_{v\in V(P)}\I{Q^v_{P}}(x)-\sum_{e\in \Eb(P)}\I{H^e_{P}}(x)+\sum_{v\in \set{a,b}}\I{Q^v_{P'}}(x)-\sum_{e\in \set{s_a,s_b}}\I{H^e_{P'}}(x)\\
        &\quad+2-3-n_a(P)+1\\
        &= \sum_{v\in V(P)}\I{Q^v_P}(x)-\sum_{e\in \Eb(P)}\I{H^e_P}(x)-n_a(P),
    \end{align*}
    which is \eqref{eq:lem_merge_arcs_lhs}.
\end{proof}

\begin{corollary}\label{cor:arc_reduction_to_cycle}
    Let $P$ be a polygon whose boundary consists of $n_a(P) \geq 1$ polygonal arcs without lines. 
    For any $x\in\RR^2$ in $P$-general position, there is a bounded polygon $P'$ such that $\partial P'$ is a cycle, such that $x$ is in $P'$-general position, and such that both sides of \eqref{eq:inside_outside_arcs} are preserved at $x$.
\end{corollary}

\begin{proof}
    This follows directly from \autoref{lem:close_end} by induction.
\end{proof}

We are now in position to prove \autoref{lem:InsideOutsideArcs}.

\begin{proof}[Proof of \autoref{lem:InsideOutsideArcs}]
By \autoref{lem:splitLine}, we can assume that $E(P)$ contains no lines, otherwise, we split them into two rays. Therefore, we have to prove
\begin{equation}\label{eq:inside_outside_arcs_noLines}
	\sum_{v\in V(P)}\I{Q^v_P}(x)-\sum_{e\in \Eb(P)}\I{H^e_P}(x)-n_a(P)
	= \I{P}(x)-1.
\end{equation}
Let $x\in\RR^2$ in $P$-general position be fixed.
Then, we can apply \autoref{cor:arc_reduction_to_cycle}, and get
a bounded polygon $P'$ whose boundary is a cycle, and with $x$ in $P'$-general position, such that
\begin{equation}\label{eq:arcs_cycle_reduction_rhs}
    \I{P}(x)-1=\I{P'}(x)-1,
\end{equation}
and 
\begin{multline}\label{eq:arcs_cycle_reduction_lhs}
    \sum_{v\in V(P)}\I{Q^v_P}(x)-\sum_{e\in \Eb(P)}\I{H^e_P}(x)-n_a(P)\\
     = \sum_{v\in V(P')}\I{Q^v_{P'}}(x)-\sum_{e\in \Eb(P')}\I{H^e_{P'}}(x).
\end{multline}
Since $\partial P'$ is a single polygonal cycle, and $x$ is in $P'$-general position, we can apply \autoref{lem:InsideOutsideBounded} to the right hand side of \eqref{eq:arcs_cycle_reduction_lhs}, and get
\begin{equation*}
    \sum_{v\in V(P)}\I{Q^v_P}(x)-\sum_{e\in \Eb(P)}\I{H^e_P}(x)-n_a(P)
     = n_h(P')+\I{P'}(x)-1.
\end{equation*}
Since $P'$ is bounded and has only one boundary component, $P'$ has no holes, i.e. $n_h(P')=0$. Together with \eqref{eq:arcs_cycle_reduction_rhs}, this proves \eqref{eq:inside_outside_arcs_noLines}, and thus \autoref{lem:InsideOutsideArcs}.
\end{proof}

\subsubsection{General case}\label{sec:inside_outside_final}
Finally, the remaining cases of \autoref{lem:InsideOutside} can be proved using the results from the previous sections.
First of all, the next lemma notes that all cycles of a polygon must describe holes except for possibly one.
\begin{lemma}\label{lem:at_most_1_outer_cycle}
    Let $P$ be a polygon and let $\gamma_1$ and $\gamma_2$ be distinct polygonal cycles that are boundary components of $P$. 
    Then, at least one $\gamma_i$ is a hole.
\end{lemma}
\begin{proof}
	Suppose that neither $\gamma_1$ nor $\gamma_2$ is a hole. Then, $P\subseteq C_i:=\cl{\interior(\gamma_i)}$ for $i=1,2$.
	In particular, $\gamma_2\subset C_1$, and since $C_1$ is homeomorphic to a closed disk, it is simply connected and therefore also $C_2\subseteq C_1$.
	Identically, one can show $C_1\subseteq C_2$. This means that $C_1=C_2$ and thus $\gamma_1=\gamma_2$, which is a contradiction.
\end{proof}
For a general polygon $P$, multiple boundary components may intersect in a common vertex.
The following lemma provides a way to describe the $P$-side of a vertex in terms of simpler components.
\begin{lemma}\label{lem:VertexOverlaps}
    Let $P$ be a polygon, and $v\in V(P)$. %
    Let $\gamma_1,\dots,\gamma_d$ be the $d\in \NN$ boundary components of $P$ that contain $v$, i.e. $v\in V(\gamma_k)$ for all $k\in[d]$.
    Assume that there are polygons $P_1,\dots,P_d$ such that $\gamma_k$ is the only boundary component of $P_k$ that contains $v$. 
    If $P\subseteq P_k$ for all $k\in[d]$, then
    \begin{equation}\label{eq:sum_cones}
        \sum_{k\in[d]}\I{Q_{P_k}^v} = d-1+\I{Q^v_P}.
    \end{equation}
\end{lemma}

\begin{proof}	
    \begin{figure}[h]
        \centering
        \def\rBR{2.5}     %
\def\rlab{.07}   %

\def\aa{10}
\def\ab{50}
\def\ac{90}
\def\ad{150}
\def\ae{210}
\def\af{320}

\definecolor{c1}{HTML}{82B366}
\definecolor{c2}{HTML}{B85450}
\definecolor{c3}{HTML}{6C8EBF}

\begin{tabular}{C{3.4cm} C{2.6cm} C{2.6cm} C{2.6cm}}
\tikzsetnextfilename{P_side_vertex_star}
\begin{tikzpicture}[very thick]
    \coordinate (v) at (0,0);
    \coordinate (e1) at (\aa:1.5);
    \coordinate (e2) at (\ab:1.3);
    \coordinate (e3) at (\ac:1.1);
    \coordinate (e4) at (\ad:1.1);
    \coordinate (e5) at (\ae:1.5);
    \coordinate (e6) at (\af:1.2);
    \coordinate (o1) at (\aa:1.5);
    \coordinate (o2) at (65:1.5);
    \coordinate (o3) at (65:1.5);
    \coordinate (o4) at (180:1.5);
    \coordinate (o5) at (\ae:1.5);
    \coordinate (o6) at (305:1.5);
    
    \draw[c3] (v) -- (e1) ;
    \draw[c1] (v) -- (e2) ;
    \draw[c1] (v) -- (e3) ;
    \draw[c2] (v) -- (e4) ;
    \draw[c2] (v) -- (e5) ;
    \draw[c3] (v) -- (e6) ;
    \draw[c3] (e1) -- (o1) ;
    \draw[c1] (e2) -- (o2) ;
    \draw[c1] (e3) -- (o3) ;
    \draw[c2] (e4) -- (o4) ;
    \draw[c2] (e5) -- (o5) ;
    \draw[c3] (e6) -- (o6) ;

    \begin{scope}[on background layer]
        \draw[fill_P_side, rounded corners=.4pt, decoration={random steps,segment length=1.3pt,amplitude=1pt}] 
            (v) -- (e1) -- (o1)  decorate { arc (\aa:65:1.5) } -- (e2) --  cycle;
        \draw[fill_P_side, rounded corners=.4pt, decoration={random steps,segment length=1.3pt,amplitude=1pt}] 
            (v) -- (e3) -- (o3)  decorate { arc (65:180:1.5) } -- (e4) --  cycle;
        \draw[fill_P_side, rounded corners=.4pt, decoration={random steps,segment length=1.3pt,amplitude=1pt}] 
            (v) -- (e5) -- (o5)  decorate { arc (\ae:305:1.5) } -- (e6) --  cycle;
    \end{scope}

    \begin{scope}[on background layer]
        \draw[pattern={Dots[angle=45,distance={4pt/sqrt(2)},radius=.4pt]}, draw=none]  
                (v) -- (\aa:1.5) arc (\aa:\ab:1.5) -- cycle;
        \draw[pattern={Dots[angle=45,distance={4pt/sqrt(2)},radius=.4pt]}, draw=none]  
                (v) -- (\ac:1.5) arc (\ac:\ad:1.5) -- cycle;
        \draw[pattern={Dots[angle=45,distance={4pt/sqrt(2)},radius=.4pt]}, draw=none]  
                (v) -- (\ae:1.5) arc (\ae:\af:1.5) -- cycle;
    \end{scope}

    \fill (v) circle (\rlab) node[left, inner sep=4pt] {$v$};

    \draw[thick,dashed] (v) circle (1) ;
    
    \node[right, inner sep=1pt] at (o3) {$\gamma_1$};
    \node[right] at (o4) {$\gamma_2$};
    \node[right] at (e6) {$\gamma_3$};

\end{tikzpicture}
&
\tikzsetnextfilename{pacman_P1}
\begin{tikzpicture}[very thick]
    \coordinate (v) at (0,0);
    \coordinate (e2) at (\ab:1);
    \coordinate (e3) at (\ac:1);
    
    \draw[c1] (v) -- (e2) ;
    \draw[c1] (v) -- (e3) ;

    \begin{scope}[on background layer]
        \draw[fill_P_side]%
        (v) -- (e3) decorate { arc (90:410:1) } -- cycle;        
    \end{scope}

    \draw[thick,dashed] (v) circle (1) ;

    \fill (v) circle (\rlab) node[left, inner sep=4pt] {$v$};

    \node[] at (300:.5) {$P_1$};%
\end{tikzpicture}
&
\tikzsetnextfilename{pacman_P2}
\begin{tikzpicture}[very thick]
    \coordinate (v) at (0,0);
    \coordinate (e4) at (\ad:1);
    \coordinate (e5) at (\ae:1);
    
    \draw[c2] (v) -- (e4) ;
    \draw[c2] (v) -- (e5) ;

    \begin{scope}[on background layer]
        \draw[fill_P_side]%
        (v) -- (e5) decorate { arc (210:510:1) } -- cycle;        
    \end{scope}

    \draw[thick,dashed] (v) circle (1) ;

    \fill (v) circle (\rlab) node[left, inner sep=4pt] {$v$};

    \node[] at (0:.5) {$P_2$}; %
\end{tikzpicture}
&
\tikzsetnextfilename{pacman_P3}
\begin{tikzpicture}[very thick]
    \coordinate (v) at (0,0);
    \coordinate (e6) at (\af:1);
    \coordinate (e1) at (\aa:1);
    
    \draw[c3] (v) -- (e6) ;
    \draw[c3] (v) -- (e1) ;

    \begin{scope}[on background layer]
        \draw[fill_P_side]%
        (v) -- (e1) decorate { arc (10:320:1) } -- cycle;        
    \end{scope}

    \draw[thick,dashed] (v) circle (1) ;

    \fill (v) circle (\rlab) node[left, inner sep=4pt] {$v$};

    \node[] at (110:.5) {$P_3$};%
\end{tikzpicture}
\end{tabular}
        \caption{
        Three boundary components $\gamma_1,\gamma_2,\gamma_3$ of $P$ (grey) intersect at vertex $v$, thus the $P$-side $Q^v_P$ (dotted), as well as its complement, are split into three circular sectors. In the interior of the dashed circle, $Q^v_P$ and $P$ agree. Since $P\subset P_k$, the two edges of component $\gamma_k$ must enclose a sector of the complement.
        }
        \label{fig:vertexStar}
    \end{figure}
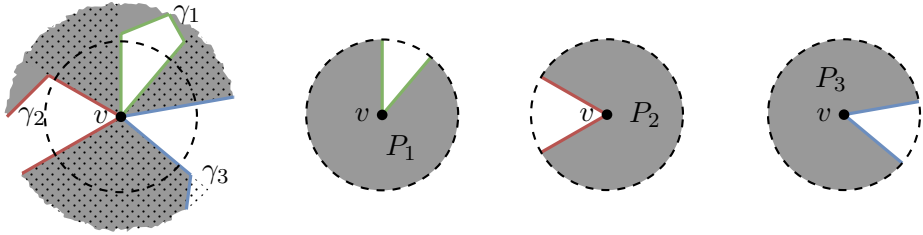
    By the definition of the $P$-side, it is sufficient to consider \eqref{eq:sum_cones} on a disk around $v$. 
    Let $D$ be a small disk centered at $v$ such that $D$ intersects only those edges of $P$ and $P_k$, $k\in[d]$, that contain $v$, cf. \autoref{fig:vertexStar}.
    Then, by the local properties of $P$-sides, we have $Q^v_P\cap D= P\cap D$ and $Q^v_{P_k}\cap D= P_{k}\cap D$.
    Thus, it suffices to show that \begin{equation}\label{eq:sum_cones_disk}
        \sum_{k\in[d]}\I{P_{k}}(x) = d-1+\I{P}(x)\quad \forall x\in D.
    \end{equation}

    The edges of $P$ incident with $v$ subdivide $D$ into circular sectors whose interior is alternately contained in $P\cap D$ and $P^c\cap D$.
    There are exactly two
    incident edges per $\gamma_k$, i.e., there are $2d$ sectors.
    Additionally, the edges of $\gamma_k$ define two sectors, namely $P_k\cap D$ and $P_k^c\cap D$, because $\gamma_k$ is the only boundary component of $P_k$ containing $v$. 
    Since $P\subseteq P_k$, it follows that $P_k^c\cap D\subseteq P^c\cap D$. Thus, $P_k^c\cap D$ equals one of the sectors of $P^c\cap D$, and $P_k\cap D$ contains all the other sectors.
     
    Given that the $\gamma_k$ are disjoint (except at $v$), each of the $d$ sectors belonging to $P^c\cap D$ equals $P_k^c\cap D$ for exactly one $k\in [d]$ and is contained in $P_l$ for all $l\in[d]\setminus\set{k}$.
    Therefore, $\sum_{k\in[d]}\I{P_{k}}(x)=d-1$ for all $x\in P^c\cap D$.
    Since $P\subseteq P_k$, it directly follows that $\sum_{k\in[d]}\I{P_{k}}(x)=d$ for all $x\in P\cap D$, which proves \eqref{eq:sum_cones_disk}.
\end{proof}

Now that the necessary preliminaries are in place, the proof of the main lemma follows.

\begin{proof}[Proof of \autoref{lem:InsideOutside}]
To simplify notation, the $P$-dependency of the number of holes is dropped by defining $h:=n_h(P)$. Let $\gamma_1,\dots,\gamma_h$ be the boundary components of $P$ that describe holes.
Then, \begin{equation*}%
    P':=P\cup \bigcup_{k\in[h]} \interior(\gamma_k)
\end{equation*}
is a polygon that satisfies
\begin{equation}\label{eq:Indicators_P'}
    \I{P'}-1=-n_a(P)+\sum_{v\in V(P')}\I{Q_{P'}^v}+\sum_{e\in E_l(P')}\I{H_{P'}^e}-\sum_{e\in E_b(P')}\I{H_{P'}^e}
\end{equation}
for all $x$ in $P'$-general position.
Indeed, that $P'$ is a polygon is trivial, and \eqref{eq:Indicators_P'} can be checked for the following cases separately:
\begin{enumerate}
    \item Case: $P$ is unbounded and contains some arcs as boundary components. \\
    Since $P$ is unbounded, any polygonal cycle of its boundary must be a hole. This means that all boundary components that are left in $P'$ are polygonal arcs. With $n_a(P)=n_a(P')$, \autoref{lem:InsideOutsideArcs} gives \eqref{eq:Indicators_P'}.
    \item Case: $P$ is unbounded but does not contain any arcs in the boundary.\\
    In this case, all boundary components are holes and $n_a(P)=0$. Therefore, we have $P'=\RR^2$ with $V(P')=E(P')=\emptyset$ and $\I{P'} = 1$, verifying \eqref{eq:Indicators_P'}.
    \item Case: $P$ is bounded.\\
    As $P$ cannot contain an arc as a boundary component, all boundary components are cycles. If all of those cycles would be holes, $P$ would be unbounded, i.e. there must be at least one cycle $\gamma_o$ with $P\subseteq\cl{\interior(\gamma_o)}$. Due to \autoref{lem:at_most_1_outer_cycle}, all other cycles must be holes.
    Therefore, $P'=\cl{\interior(\gamma_o)}$, and \eqref{eq:Indicators_P'} follows from \autoref{lem:InsideOutsideBounded} since $n_a(P)=0$, $n_h(P')=0$, and $E_l(P')=\emptyset$.
\end{enumerate}
To keep the notation as simple as possible, we consider all following equations involving indicator functions to be restricted to $x\in\RR^2$ in $P$-general position. These $x$ are then also in $P'$-general position because $E(P')\subset E(P)$.
We have, $P=P'\setminus\bigcup_{k\in[h]}\interior(\gamma_k)$.
Since $\interior(\gamma_k)$ and $\interior(\gamma_l)$ are disjoint for two holes with $k\neq l$, we get
\begin{align*}
    \I{P}-1 
    &= \I{P'}-\sum_{k\in[h]}\I{\interior(\gamma_k)} -1
    = \I{P'}-1+\sum_{k\in[h]}(\I{\cl{\exterior(\gamma_k)}}-1) \\
    &= -n_a(P)+\sum_{v\in V(P')}\I{Q_{P'}^v} + \sum_{e\in E_l(P')}\I{H_{P'}^e} - \sum_{e\in E_b(P')}\I{H_{P'}^e} 
    \\&\qquad+\sum_{k\in[h]}(\I{\cl{\exterior(\gamma_k)}}-1),
\end{align*}
where we have expressed $\I{P'}-1$ using \eqref{eq:Indicators_P'}.
$P_k:=\cl{\exterior{\gamma_k}}$ are polygons with $\partial P_k=\gamma_k$ and $n_h(P_k)=1$. 
Note, that the $P_k$-side of an edge $e\in E(\gamma_k)$ agrees with its $P$-side, and the same holds for the $P'$-sides of $e\in E(P')$. Then, expressing the summands $\I{\cl{\exterior(\gamma_k)}}-1$ using \autoref{lem:InsideOutsideBounded}, yields
\begin{align}
    \I{P}-1 &= -n_a(P)+\sum_{v\in V(P')}\I{Q_{P'}^v} + \sum_{e\in E_l(P')}\I{H^e_P} - \sum_{e\in E_b(P')}\I{H^e_P} \nonumber
    \\&\qquad+ \sum_{k\in[h]}\big(\sum_{v\in V(\gamma_k)}\I{Q_{P_k}^v} - \sum_{e\in E(\gamma_k)}\I{H^e_P}-1\big) \nonumber 
    \\&= -n_a(P)-h+\sum_{v\in V(P')}\I{Q_{P'}^v}+\sum_{e\in E_l(P)}\I{H^e_P} -\sum_{e\in E_b(P)}\I{H^e_P}\nonumber
    \\&\qquad+ \sum_{k\in[h]}\sum_{v\in V(\gamma_k)}\I{Q_{P_k}^v},  \label{eq:Indicator_P_tmp}
\end{align}
where we used that $E_b(P)=E_b(P')\cup \bigcup_{k\in[h]}E(\gamma_k)$ is a disjoint union, and $E_l(P')=E_l(P)$.
With $K(v):=\set{R\in\set{P',P_1,\dots,P_h}:v\in V(R)}$, we can write
\begin{equation}\label{eq:sum_P_sides_v_masked}
	\sum_{v\in V(P')}\I{Q_{P'}^v} + \sum_{k\in[h]}\sum_{v\in V(\gamma_k)}\I{Q_{P_k}^v} = \sum_{v\in V(P)}\sum_{R\in K(v)} \I{Q_{R}^v}.
\end{equation}
Since $P\subseteq P'$, and $P\subseteq P_k$ for all $k\in[h]$, and because $P'$ has at most one boundary component containing $v\in V(P)$ by \autoref{lem:arcInterstions} and \autoref{lem:at_most_1_outer_cycle}, we can apply \autoref{lem:VertexOverlaps} to \eqref{eq:sum_P_sides_v_masked} to get
\begin{equation*}
	\sum_{v\in V(P)}\sum_{R\in K(v)} \I{Q_{R}^v} =
	\sum_{v\in V(P)}(\abs{K(v)}-1+\I{Q^v_P}) = d(P)+\sum_{v\in V(P)}\I{Q^v_P}. 
\end{equation*}
Here, we have used that $|K(v)|=\deg_P(v)/2$, since every boundary component has two incident edges per vertex.
Substituting this into \eqref{eq:sum_P_sides_v_masked}, together with \eqref{eq:Indicator_P_tmp}, gives
\begin{equation*}
    \I{P}-1 = -n_a(P)-h+d(P) + \sum_{v\in V(P)} \I{Q^v_P} + \sum_{e\in E_l(P)}\I{H^e_P} - \sum_{e\in E_b(P)} \I{H^e_P}.
\end{equation*}
Adding $1$ to both sides and substituting $c(P)=1+d(P)-n_a(P)-h$ yields the desired result.
\end{proof}

	\section{max-Representation of CPA Functions}
The results of this section reduce the representation of $\CPA$ functions, given by \autoref{lem:CPA_Decomp}, to a sum of nested signed maxima.
For a vertex $v\in V(f)$, its \emph{degree} $\deg(v):=|\set{e\in E(f):v\in e}|$ is defined as the number of edges that contain $v$ as a vertex.
A $\CPA$ function is called \emph{$v$-function} if $v\in\RR^2$ is its only vertex and all its edges are rays.
Note, that $f\in\CPA_{\deg(v)}$ for every $v$-function $f$, and that any $f\in\CPL$ is a $0$-function.

In general, the number of vertices and edges of a $\CPA_p$ function is not bounded by $p$,
since arbitrary edges can be split by adding vertices. Consequently, the number of summands in the representation from \autoref{lem:CPA_Decomp} is not known. However, there always exists a sparse choice of vertices, as described in the following lemma.
\begin{lemma}\label{lem:sparse_P_V_and_E}
	For every $f\in\CPA_p$, there exists an admissible set of pieces $\cP$, with corresponding sets of vertices $V$ and edges $E$, such that $|\cP|\leq p$, $\deg(v)\geq 3$ for all $v\in V$, and $|E|\leq 3p$. 
\end{lemma}
\begin{proof}
	Let $f\in\CPA_p$ be an arbitrary function. By definition, there is an admissible set of pieces $\cP(f)$ such that $|\cP(f)|=p$.
	Let $V(f)$ and $E(f)$ be the corresponding sets of vertices and edges.
	The homeomorphisms of the boundary components of a polygon ensure that $\deg(v)\geq 2$ for all $v\in V(f)$.
	
	If there exists some $v\in V(f)$ with $\deg(v)=2$, then there must be two pieces $P,Q\in\cP(f)$ that share the two incident edges. Let $f_P$ and $f_Q$ be the respective affine components of $P$ and $Q$.
	If $f_P\neq f_Q$, then the incident edges lie on the same line, and the vertex can be removed from $V(f)$ by merging the incident edges.
	If $f_P=f_Q$, then the pieces  $P$ and $Q$ can be replaced by the single piece $P\cup Q$, effectively removing $v$ from $V(f)$, removing the incident edges from $E(f)$, and reducing the number of pieces.	%
			
	By iterating this process for all vertices $v\in V(f)$ with $\deg(v) = 2$, we obtain a new set $\cP'$ of $p' \leq p$ admissible pieces for $f$, along with corresponding vertex and edge sets $V'$ and $E'$, where $\deg(v) \geq 3$ for all $v \in V'$. 
	
	The sets $\cP'$, $V'$, and $E'$ are the desired sets. It is left to prove that $|E'|\leq 3p$.
	
	Since all vertices in $V'$ have degree at least three, we have
	\begin{equation}\label{eq:double_counting}
		3|V'|\leq \sum_{v\in V'}\deg(v)\leq 2|E'|,
	\end{equation} 
	where the last inequality holds as an equality if all edges are line segments.
	Let $D$ be a disk that contains all vertices in $V'$ and intersects every ray and line in $E'$.
	We form a plane graph $\cl{G}:=(\cl{V},\cl{E})$ from $(V',E')$ by adding an auxiliary vertex outside of $D$ and connecting all unbounded edges to this vertex. Hence, $|\cl{V}|=|V'|+1$, $|\cl{E}|=|E'|$, and $\cl{G}$ has $p'$ faces.
	Let $k$ be the number of connected components of $\cl{G}$. 
	A straightforward extension of Euler's formula for disconnected graphs (see, e.g., \cite[Thm. 4.2.9]{Diestel2017} for the standard version) is given by
	\begin{equation}\label{eq:Euler_conn_comp}
		\abs{\cl{V}}-\abs{\cl{E}}+p'=k+1.
	\end{equation}
	For $(V',E')$, \eqref{eq:Euler_conn_comp} gives
	\begin{equation*}%
		\abs{V'}-\abs{E'}+p'=k\geq 0.
	\end{equation*}	
	Combining this with \eqref{eq:double_counting} yields 
	\begin{equation*}
	\abs{E'}\leq 3p'\leq 3p.
	\end{equation*}
\end{proof}

The first step to derive the $\max$-representation of $\CPA$ functions is \autoref{lem:fv_reduction}, which can be applied to the functions $f^v$.
It states that every $v$-function can be expressed as a sum of $v$-functions with three pieces each.

\begin{lemma}%
\label{lem:fv_reduction}
    For $v\in\RR^2$, let $f\in\CPA_p$ be an arbitrary $v$-function with $p\geq 3$ pieces.
    Then, $f$ can be written as \begin{equation*}%
        f = \sum_{n=1}^{p-2} f^{(n)},
    \end{equation*}
    where $f^{(n)}\in\CPA_3$ is a $v$-function for every $n\in[p-2]$.
\end{lemma}

In the proof of \autoref{lem:fv_reduction}, I will use an auxiliary lemma that describes how to merge two adjacent pieces of a $\CPL$ function by subtracting a function $f^{(n)}$ with three pieces. Two pieces are called adjacent if their intersection contains a line segment.

    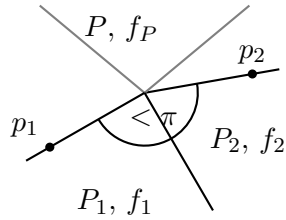
\begin{figure}[h!]
        \centering
        \tikzsetnextfilename{CPLskeleton}
\begin{tikzpicture}[thick, scale=.6]
    \coordinate (0) at (0,0);
    \coordinate (1) at (10:3);
    \coordinate (2) at (40:3);
    \coordinate (3) at (140:3);
    \coordinate (4) at (210:3);
    \coordinate (5) at (300:3);

    \draw (0) -- (1);
    \draw[gray] (0) -- (2);
    \draw[gray] (0) -- (3);
    \draw (0) -- (4);
    \draw (0) -- (5);

    \pic ["$<\pi$", draw, angle radius=.7cm, angle eccentricity=.5] {angle=4--0--1}; 

    \fill ($(0)!0.8!(4)$) circle (.1) node[above left] {$p_1$};
    \fill ($(0)!0.8!(1)$) circle (.1) node[above] {$p_2$};

    \node[] at (335:2.5) {$P_2$, $f_2$};
    \node[] at (255:2.5) {$P_1$, $f_1$};
    \node[] at (110:1.5) {$P$, $f_P$};
    
\end{tikzpicture}
        \caption{
        Pieces (enclosed by solid lines) of a $\CPL$ function $f$ with adjacent pieces $P_1$ and $P_2$.
        One can define a $3$-piece function that equals $f$ on $P_1$ and $P_2$ by interpolating $f_1$ and $f_2$ on their outer rays.
        }
        \label{fig:CPLskeleton}
    \end{figure}

\begin{lemma}%
\label{lem:CPL_merge_pieces}
    Let $f\in\CPL_p$ be a piecewise linear function with $p\geq 3$ pieces, such that there are two adjacent pieces that enclose an angle less than $\pi$.
    Then, there exists a function $f'\in\CPL_3$ such that $f-f'\in \CPL_{p-1}$.
\end{lemma}

\begin{proof} 
Let $\set{P_i}_{i\in[p]}$ be the pieces of $f$, such that $P_1$ and $P_2$ are adjacent pieces that enclose an angle less than $\pi$, and let $\set{f_i}_{i\in[p]}$ be the corresponding linear components.
The boundaries of the pieces are each the union of two rays with vertex $0$. As $P_1$ and $P_2$ are adjacent, they share one of these rays.
Let $p_1,p_2\in\RR^2\setminus\set{0}$ be points on the non-shared ray of $P_1$ and $P_2$, respectively, as shown in \autoref{fig:CPLskeleton}.
As the angle at $0$ enclosed by $P_1\cup P_2$ is smaller than $\pi$, $p_1$ and $p_2$ are linearly independent. 
Therefore, there exists a linear function $f_P$ such that $f_P(p_1)=f_1(p_1)$ and $f_P(p_2)=f_2(p_2)$. By linearity of $f_i$, we have $f_P=f_1$ on $\Span(p_1)$ and $f_P=f_2$ on $\Span(p_2)$.
Therefore, with $P:=\bigcup_{i=3}^pP_i$, the function
\begin{equation*}
    f'(x):=\begin{cases}
        f_1,\quad&\text{in }P_1\\
        f_2,\quad&\text{in }P_2\\
        f_P,\quad&\text{in }P\\
    \end{cases}
\end{equation*}
is continuous and piecewise linear. Since the interior of $P$ is connected, $P_1$, $P_2$, and $P$ are admissible pieces, thus $f'\in\CPL_3$.

For $i>2$, we have $(f-f')|_{P_i}=f|_{P_i}-f_P$, which is a linear function.
Moreover, $(f-f')|_{P_1}=(f-f')|_{P_2}=0$, i.e. $f-f'$ is linear on $P_1\cup P_2$. Continuity of $f-f'$ is implied by the continuity of $f$ and $f'$, and we get $f-f'\in\CPA_{p-1}$.
\end{proof}

With this preparation, we now come to the proof of \autoref{lem:fv_reduction}.

\begin{proof}[Proof of \autoref{lem:fv_reduction}]
We may assume that $f\in\CPL_p$. Otherwise, represent $\tilde{f}(x):=f(v-x)-f(v)$, which is contained in $\CPL_p$, as $\tilde{f} = \sum_{n\in[p-2]} \tilde{f}^{(n)}$ with $\tilde{f}^{(n)}\in\CPL_3$. This gives $f(x)=\sum_{n\in[p-2]} \tilde{f}^{(n)}(v-x)+f(v)$ as desired.

We will proceed by induction. For $p=3$, the statement is trivial with $f^{(1)}:=f$.

Let $p\geq 4$, and assume that the statement is proven for $p-1$.
First, assume that there are two adjacent pieces that enclose an angle smaller than $\pi$. This is always the case for $p>4$. Then, we can apply \autoref{lem:CPL_merge_pieces} and are done by induction.

Now, consider the remaining case where $p=4$, and $f$ is such that no two pieces form an angle smaller than $\pi$. Then, the pieces of $f$ are defined by two intersecting lines. Let $P_1,P_2,P_3,P_4$ be the pieces of $f$ in the order depicted in \autoref{fig:CPL_2_lines}, and $a\in(P_1\cap P_2)\setminus\set{0}$. 
Since $f_1=f_2$ on $\Span(a)$ and $(P_1\cup P_4)\cap (P_2\cup P_3)=\Span(a)$, the function defined by
\begin{equation*}
    f^{(1)}(x):=\begin{cases}
        f_1,\quad&\text{in }P_1\cup P_4\\
        f_2,\quad&\text{in }P_2\cup P_3
    \end{cases}
\end{equation*}
is a continuous function defined on two pieces that intersect in a line. In particular $f^{(1)}\in\CPL_3$.
Moreover, \begin{equation*}
    f^{(2)}(x):=f(x)-f^{(1)}(x)=\begin{cases}
        0, & x\in P_1\cup P_2 \\
        f_3(x)-f_2(x),\quad & x\in P_3 \\
        f_4(x)-f_1(x),\quad & x\in P_4 \\
    \end{cases},
\end{equation*}
is a $\CPL_3$ function (actually even $\CPL_2$).
Thus, $f=f^{(1)}+f^{(2)}$ with $f^{(1)},f^{(2)}\in\CPL_3$.
\end{proof}

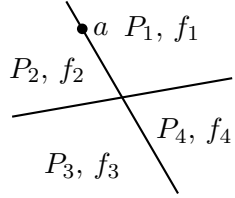
\begin{figure}[h!]
    \centering
    \def\extd{1.4}

\tikzsetnextfilename{CPL_p4_two_lines}
\begin{tikzpicture}[thick, scale=.7]
    \coordinate (b) at (120:1.5);
    \coordinate (a) at (10:1.5);
    \coordinate (ma) at (10:-1.5);

    \draw ($\extd*(a)$) -- ($-\extd*(a)$);
    \draw ($\extd*(b)$) -- ($-\extd*(b)$);

    \fill (b) circle (.1) node[right] {$a$};
    
    \node[] at (60:1.5) {$P_1$, $f_1$};
    \node[] at (240:1.5) {$P_3$, $f_3$};
    \node[] at (335:1.5) {$P_4$, $f_4$};
    \node[] at (160:1.5) {$P_2$, $f_2$};
    
\end{tikzpicture}
    \caption{A four-piece $\CPL$ function with no two adjacent pieces forming an angle smaller than $\pi$.}
    \label{fig:CPL_2_lines}
\end{figure}

I will now give an explicit expression for $v$-functions with three pieces.

\begin{lemma}\label{lem:3piece_analytical}
    Let $f\in\CPA_3$ be a $v$-function with three pieces. 
    Then, there are signs $\sigma_1,\sigma_2\in\set{-1,1}$, such that $f$ can be expressed as 
    \begin{equation}\label{eq:3piece_analytical}
        f=\sigma_1\max(g_1,\sigma_2\max(g_2,g_3)),
    \end{equation}
    where $g_i$ are affine components of $f$ multiplied by $-1$ or $1$.
\end{lemma}
\begin{proof}
    First, I note a trivial observation.
    Let $g$ and $h$ be affine functions such that $g-h$ is not constant. 
    Then $g-h$ vanishes precisely on a line $l$ that divides $\RR^2$ into two open half-planes.
    If $x_1,x_2\in\RR^2$ are contained in the same half-plane, then the line segment $\cl{x_1x_2}$ does not intersect $l$ due to the convexity of half-planes. 
    Consequently, the continuous function $g-h$ does not change sign along $\cl{x_1x_2}$.
    However, if $x_1$ and $x_2$ lie in opposite half-planes, the segment $\cl{x_1x_2}$ intersects $l$ exactly once, causing the function to change sign.
    Therefore, $g>h$ in one half-plane, and $h>g$ in the other. 
    More generally, it holds that $g\geq h$, respectively $h\geq g$, in the closed half-planes, which also allows for the case that $g=h$ (with any line).

    \begin{figure}
        \centering
        \begin{tabular}{C{5cm} C{5cm}}
        \def\aa{38}
\def\ab{160}
\def\ac{310}
\def\r{2}
\tikzsetnextfilename{CPL_3_convex}
\begin{tikzpicture}[scale=1]
    \coordinate (0) at (0,0);
    \coordinate (A) at (\aa:\r);
    \coordinate (B) at (\ab:\r);
    \coordinate (C) at (\ac:\r);
    \draw (0) -- (A);
    \draw (0) -- (B);
    \draw (0) -- (C);
    \draw[dashed] (0) -- ($-1*(A)$);
    \draw[dashed] (0) -- ($-1*(B)$);
    \draw[dashed] (0) -- ($-1*(C)$);
    
    \pic ["$\alpha_1$", draw, angle radius=.7cm, angle eccentricity=.6] {angle=A--0--B}; 
    \pic ["$\alpha_2$", draw, angle radius=.75cm, angle eccentricity=.6] {angle=B--0--C}; 
    \pic ["$\alpha_0$", draw, angle radius=.8cm, angle eccentricity=.6] {angle=C--0--A}; 

    \fill ($(0)!0.8!(A)$) circle (.07) node[left] {$p_2$};
    \fill ($(0)!0.8!(B)$) circle (.07) node[below] {$p_0$};
    \fill ($(0)!0.8!(C)$) circle (.07) node[left] {$p_1$};
    \fill ($(0)$) circle (.07) node[below] {$v$};
    
    \node[left] at ($-.84*(A)$) {$l_2$};
    \node[below] at ($-.9*(B)$) {$l_0$};
    \node[below] at ($-.9*(C)$) {$l_1$};

    \node[] at (\aa/2+\ab/2:.6*\r) {$P_1$, $f_1$};
    \node[] at (\ab/2+\ac/2:.6*\r) {$P_2$, $f_2$};
    \node[] at (\ac/2+\aa/2+180:.7*\r) {$P_0$, $f_0$};
\end{tikzpicture} &
        \def\aa{110}
\def\ab{160}
\def\ac{245}
\def\r{2}
\tikzsetnextfilename{CPL_3_concave}
\begin{tikzpicture}[scale=1]
    \coordinate (0) at (0,0);
    \coordinate (A) at (\aa:\r);
    \coordinate (B) at (\ab:\r);
    \coordinate (C) at (\ac:\r);
    \draw (0) -- (A);
    \draw (0) -- (B);
    \draw (0) -- (C);
    \draw[dashed] (0) -- ($-1*(A)$);
    \draw[dashed] (0) -- ($-1*(B)$);
    \draw[dashed] (0) -- ($-1*(C)$);
    
    \pic ["$\alpha_1$", draw, angle radius=.7cm, angle eccentricity=.6] {angle=A--0--B}; 
    \pic ["$\alpha_2$", draw, angle radius=.75cm, angle eccentricity=.6] {angle=B--0--C}; 
    \pic ["$\alpha_0$", draw, angle radius=.8cm, angle eccentricity=.5] {angle=C--0--A}; 

    \fill ($(0)!0.8!(A)$) circle (.07) node[left] {$p_2$};
    \fill ($(0)!0.8!(B)$) circle (.07) node[below] {$p_0$};
    \fill ($(0)!0.8!(C)$) circle (.07) node[left] {$p_1$};
    \fill ($(0)$) circle (.07) node[below] {$v$};
    
    \node[left] at ($-.9*(A)$) {$l_2$};
    \node[below] at ($-.9*(B)$) {$l_0$};
    \node[right] at ($-.9*(C)$) {$l_1$};

    \node[] at (\aa/2+\ab/2:.7*\r) {$P_1$, $f_1$};
    \node[] at (\ab/2+\ac/2:.7*\r) {$P_2$, $f_2$};
    \node[] at (\ac/2+\aa/2+180:.7*\r) {\begin{tabular}{c} $P_0,$\\  $f_0$ \end{tabular}};
\end{tikzpicture}
        \end{tabular}
        \caption{$v$-functions with three pieces that are enclosed by rays (solid lines).
        The sign of $f_i-f_j$ only changes on the affine hull of the ray between $P_i$ and $P_j$ (dashed lines). Left: all angles smaller than $\pi$. Right: $\alpha_0>\pi$.}
        \label{fig:3_piece_representation}
    \end{figure}
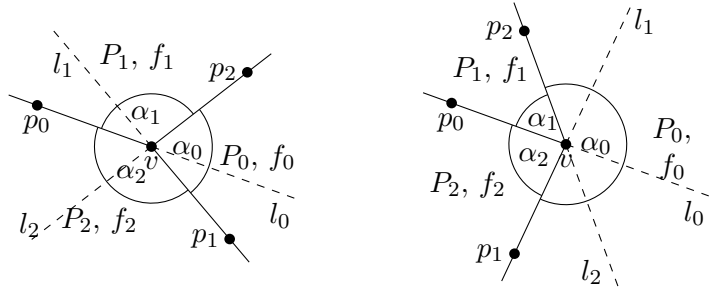

    Let $P_0,P_1,P_2$ be the pieces of $f$, and denote their interior angles by $\alpha_0,\alpha_1,\alpha_2$, cf. \autoref{fig:3_piece_representation}.
    For $i\in\set{0,1,2}$, let $f_i$ be the affine component of $P_i$.
    Let $p_0,p_1,p_2\neq 0$ be points on the ray between $P_1$ and $P_2$, $P_0$ and $P_2$, and $P_0$ and $P_1$ respectively. Further, define $l_0:=\aff(\cl{vp_0})$, $l_1:=\aff(\cl{vp_1})$, and $l_2:=\aff(\cl{vp_2})$.

    First, we consider the case that $\alpha_i<\pi$ for all $i\in\set{0,1,2}$.
    Consider the case that $f(p_0)=f_1(p_0)=f_2(p_0)\geq f_0(p_0)$. 
    Then, as $P_1$ and $p_0$ lie on the same side of $l_2$, which is the line with $f_0=f_1$, this extends to  $f_1\geq f_0$ in $P_1$. Similarly, $f_2\geq f_0$ in $P_2$.
    Moreover, $P_0$ and $p_0$ lie on opposite sides of $l_1$ and $l_2$, which implies that $f_0\geq f_1$ and $f_0\geq f_2$ in $P_0$.
    In particular, $f_2(p_1)=f_0(p_1)\geq f_1(p_1)$.
    Since $P_2$ and $p_1$ lie in the same half-plane w.r.t $l_0$, opposite to that of $P_1$, this means that $f_2\geq f_1$ in $P_2$ and $f_1\geq f_2$ in $P_1$.
    In total, we have $f_1\geq \max(f_0,f_2)$ in $P_1$, $f_2\geq \max(f_0,f_1)$ in $P_2$, and $f_0\geq \max(f_1,f_2)$ in $P_0$. Thus $f=\max(f_0,f_1,f_2)=\max(f_0,\max(f_1,f_2))$, which is of the form \eqref{eq:3piece_analytical}.
    For the case that $f(p_0) < f_0(p_0)$, one can apply the same arguments to $-f$ and get $f=-\max(-f_0,\max(-f_1,-f_2))$, which also satisfies \eqref{eq:3piece_analytical}.

    Now, consider the case that $\alpha_i\geq\pi$ for some $i\in\set{0,1,2}$, say $i=0$.
    First, assume that $f(p_0)=f_1(p_0)=f_2(p_0)\geq f_0(p_0)$. 
    Since $\alpha_1+\alpha_2<\pi$, all of $P_1\cup P_2$ lies on the same side as $p_0$ of both $l_1$ and $l_2$.
    Thus, $\min(f_1,f_2)\geq f_0$ in $P_1\cup P_2$.
    Moreover, $f_2(p_2)\geq f_0(p_2)=f_1(p_2)$, and since $l_0$ separates $P_1$ from $P_2$, we have $f|_{P_1\cup P_2} = \min(f_1,f_2)|_{P_1\cup P_2}$.
    In contrast, any $x\in P_0$ and $p_0$ lie on opposite sides of either $l_1$ or $l_2$, and therefore $f_0\geq\min(f_1,f_2)$ in $P_0$.
    In total, we get $f=\max(f_0,\min(f_1,f_2))$, which is of the form \eqref{eq:3piece_analytical}.
    Again, \eqref{eq:3piece_analytical} can also be achieved for the case that $f(p_0)<f_0(p_0)$ by considering $-f$.
\end{proof}

Together with \autoref{lem:fv_reduction}, \autoref{lem:3piece_analytical} allows for a reformulation of \autoref{lem:CPA_Decomp} in a standardised form.

\begin{theorem}%
\label{lem:CPA_reduction}
    Let $f\in\CPA_p$ be a continuous piecewise affine function.
    Then, there are affine functions $f^{(k)}_n$, and signs $\sigma_n^{(k)}\in\set{-1,1}$, such that
    \begin{equation*}%
        f(x) = \sum_{n=1}^{9p} \sigma_{n}^{(1)}\max(f_{n}^{(1)},\sigma_{n}^{(2)}\max(f_{n}^{(2)},f_{n}^{(3)})).
    \end{equation*}
\end{theorem}

\begin{proof}
Let $\cP(f)$, $V(f)$, $E(f)$ be the sets given by \autoref{lem:sparse_P_V_and_E}, and let $E_b(f), E_l(f)\subseteq E(f)$ be the subsets of line segments and lines. By \autoref{lem:CPA_Decomp}, there exists an affine function $h$ such that
\begin{equation}\label{eq:CPA_Decomp_reduction_proof}
    f(x) = \sum_{v\in V(f)} f^{{v}}(x) + \sum_{\substack{e\in E_l(f)}}f^{{e}}(x) - \sum_{\substack{e\in \Eb(f)}}f^{{e}}(x) +h.
\end{equation}
For an edge $e\in E(f)$, the function $f^e$ is a $\CPA_2$ function with two pieces that are half-planes separated by a line. If $f_1^e$ and $f_2^e$ are the affine components of $f^e$, then $f^e=\max(f_1^e,f_2^e)$ or $f^e=\min(f_1^e,f_2^e)=-\max(-f_1^e,-f_2^e)$.
Since $-f_i^e$ is also affine, we get
\begin{align}
    \sum_{\substack{e\in E_l(f)}}f^{{e}}(x) - \sum_{\substack{e\in \Eb(f)}}f^{{e}}(x) &= \sum_{n=1}^{\abs{E_l}+\abs{\Eb}} \sigma_n\max(f_{n}^{(1)}, f_{n}^{(2)}) \nonumber \\ 
    &= \sum_{n=1}^{\abs{E_l}+\abs{\Eb}} \sigma_n\max(f_n^{(1)},\max(f_{n}^{(2)}, f_{n}^{(2)})), \label{eq:E_sum_as_max}
\end{align}
for some signs $\sigma_n\in\set{-1,1}$ and affine functions $f_n^{(1)}$, $f_n^{(2)}$. 

For any $v\in V(f)$, $f^v$ is a $v$-function with $\deg(v)\geq 3$ pieces. Therefore, by \autoref{lem:fv_reduction}, we can express $f^v$ as 
\begin{equation*}
    f^v = \sum_{n=1}^{d_v} f_{v,n},
\end{equation*}
where $f_{v,n}\in\CPA_3$ for $n\in[d_v]$ are $v$-functions, and $d_v:=\deg(v)-2$.
Applying \autoref{lem:3piece_analytical} to the three piece functions $f_{v,n}$ gives $\sigma_{v,n}^{(1)},\sigma_{v,n}^{(2)}\in\set{-1,1}$ and affine functions $f_{v,n}^{(1)},f_{v,n}^{(2)},f_{v,n}^{(3)}$ such that 
\begin{equation*}
    f_{v,n}= \sigma_{v,n}^{(1)}\max(f_{v,n}^{(1)},\sigma_{v,n}^{(2)}\max(f_{v,n}^{(2)},f_{v,n}^{(3)})).
\end{equation*}
Therefore,
\begin{equation}\label{eq:V_sum_as_max}
    f^v = \sum_{n=1}^{d_v} \sigma_{v,n}^{(1)}\max(f_{v,n}^{(1)},\sigma_{v,n}^{(2)}\max(f_{v,n}^{(2)},f_{v,n}^{(3)})).
\end{equation}

Note, that $\max(h_1,h_2)+h=\max(h_1+h,h_2+h)$, and $h_i+h$ is affine if $h_i$ is. 
Therefore, the affine function $h$ can be integrated into one of the summands of either \eqref{eq:E_sum_as_max} or \eqref{eq:V_sum_as_max}.
Thus, using \eqref{eq:E_sum_as_max}, and \eqref{eq:V_sum_as_max}, we can write \eqref{eq:CPA_Decomp_reduction_proof} as
\begin{align*}
    f(x)&=\sum_{v\in V}\sum_{n=1}^{d_v} \sigma_{v,n}^{(1)}\max(f_{v,n}^{(1)},\sigma_{v,n}^{(2)}\max(f_{v,n}^{(2)},f_{v,n}^{(3)})) \nonumber\\
        &+\sum_{n=1}^{\abs{E_l}+\abs{\Eb}} \sigma_n\max(f_n^{(1)},\max(f_{n}^{(2)}, f_{n}^{(2)})).
\end{align*}

It remains to show that $\sum_{v\in V}d_v+|E_l|+|\Eb|\leq 9p$.
Since we chose the vertex and edge sets according to \autoref{lem:sparse_P_V_and_E}, we have $|E|\leq 3p$.
This directly implies $|E_l|+|\Eb|\leq 3p$, since $E_l\cup \Eb \subseteq E$ is a disjoint union.
Moreover, \begin{equation*}
    \sum_{v\in V} d_v \leq \sum_{v\in V}\deg(v) \leq 2\abs{E}\leq 6p,
\end{equation*}
which completes the proof.
\end{proof}

	\section{Neural Network Representation of CPA Functions}%
The following theorem translates the $\max$-representation of $\CPA$ functions, provided in \autoref{lem:CPA_reduction}, into a neural network representation.

\begin{theorem}%
	\label{thm:NN_representation}
    Any function $f\in\CPA_p$ with $p$ pieces can be represented by a $\relu$ neural network of depth $3$ with width vector $(s_1,s_2)\leq (45p,27p)$.
\end{theorem}

The proof relies on the following definition, which can be interpreted as the direct sum of affine functions.

\begin{definition}
	For $w\in\NN$ and $i\in[w]$, let $T_i:\RR^{n_i}\to \RR^{m_i}$ be affine functions.
	With $n:=\sum_{i\in[w]} n_i$ and $m:=\sum_{i\in[w]} m_i$, define
	\begin{equation*}
		T:=\cvb{T_1\\ \vdots\\ T_w}:\RR^{n}\to\RR^{m},\quad T(x_1,\dots,x_w) := \cv{T_1(x_1)\\\vdots\\ T_w(x_w)}\text{ for }x_i\in\RR^{n_i}.
	\end{equation*}
	The function $T$ is referred to as the affine function that \emph{stacks} $T_1,\dots,T_w$.
\end{definition}

\begin{proof}[Proof of \autoref{thm:NN_representation}]
Using \autoref{lem:CPA_reduction}, $f$ can be expressed as
\begin{equation}\label{eq:f_decomposed_NN_proof}
	f(x) = \sum_{n=1}^{9p} \sigma_{n}^{(1)}\max(f_{n}^{(1)},\sigma_{n}^{(2)}\max(f_{n}^{(2)},f_{n}^{(3)})),
\end{equation}
where $f_n^{(k)}$ are affine functions and $\sigma_n^{(k)}\in\set{-1,1}$.
To obtain a neural network representation of $f$, we first construct individual neural networks for each summand in \eqref{eq:f_decomposed_NN_proof} and then stack the affine functions corresponding to the same layer.

Let $f_1,f_2,f_3$ be affine functions.
Using the $\relu$ activation function $\rho$, one can 'skip' the affine function $f_1$ through one layer using 
\begin{equation*}
	f_1 = \max(0,f_1)-\max(0,-f_1) = \rho(f_1) - \rho(-f_1).
\end{equation*}
This can be written as $f_1=T^{(2)}_1\circ\rho\circ T^{(1)}_1$, where $T^{(1)}_1=(f_1,-f_1)^T$, and $T^{(2)}_1(x)=(1,-1)x$, which is a neural network with one hidden layer of width two.
Next, for the function $\max(f_2,f_3)$, we have 
\begin{equation*}
    \max(f_2,f_3) = \max(0, f_2-f_3) + f_3 = \rho(f_2-f_3) + \rho(f_3) - \rho(-f_3),
\end{equation*}
which allows us to represent $\sigma_2\max(f_2,f_3)$ by a neural network with one hidden layer of width three, given by $T^{(1)}_2=(f_2-f_3,f_3,-f_3)^T$ and $T^{(2)}_2(x)=\sigma_2(1,1,-1)x$. 
Similarly, $\sigma_1\max(\blank,\blank)$ can be represented by a neural network $(T^{(1)}_3,T^{(2)}_3)$ of the same width.

In total, we get
\begin{equation}\label{eq:NN_single_summand}
    \sigma_1\max(f_1,\sigma_2\max(f_2,f_3)) = T^{(2)}_3\circ \rho\circ \underbrace{ T^{(1)}_3\circ \cvb{T^{(2)}_1\\ T^{(2)}_2}}_{\RR^5\to\RR^3} \circ \rho\circ \underbrace{\cvb{T^{(1)}_1\\T^{(1)}_2}}_{\RR^2\to\RR^5}.
\end{equation}
Therefore, each summand of the form $\sigma_{n}^{(1)}\max(f_{n}^{(1)},\sigma_{n}^{(2)}\max(f_{n}^{(2)},f_{n}^{(3)}))$ can be represented by a neural network with two hidden layers of width $5$ and $3$ respectively.

Now, we can represent a vector-valued function, whose $i$-th component is the $i$-th summand of \eqref{eq:f_decomposed_NN_proof}, with a neural network by stacking the affine functions of the neural networks that represent the individual summands for each layer separately. Finally, to represent $f$, we need to replace the affine function of the last layer with its concatenation with the linear sum operation.
Since there are $9p$ summands, the width vector is $(s_1,s_2)=9p(5,3)=(45p,27p)$.
\end{proof}

The provided constants are meant to show that they are reasonable and could be slightly improved.
The following corollary highlights the efficiency of the neural network representation by addressing the sparsity of the model parameters.

\begin{corollary}
	\label{cor:sparsity}
	Any $f\in\CPA_p$ with $p$ pieces can be represented by a $\relu$ neural network of depth $3$ with $O(p)$ non-zero parameters.
\end{corollary}
\begin{proof}
	For $i\in[w]$, let $T_i:\,\RR^{n_i}\to \RR^{m_i}$ be affine functions, with $T_i(x)=A_ix+b_i$ for some $A_i\in\RR^{m_i\times n_i}$ and $b_i\in\RR^{m_i}$.
	Then, we have
	\begin{equation*}
		\cvb{T_1\\\vdots\\T_w}(x) = \begin{pmatrix}A_1 & 0 & 0\\ 0 & \ddots & 0 \\ 0 & 0 & A_w\end{pmatrix}x + \cv{b_1\\\vdots\\ b_w},\quad\forall x\in\RR^{n_1+\dots+n_w}.
	\end{equation*}
	Thus, since \autoref{thm:NN_representation} stacks $9p$ affine functions of the same form given by \eqref{eq:NN_single_summand} for each layer, the total number of non-zero parameters is $9p$ times the number of parameters required in \eqref{eq:NN_single_summand}.
	We have $A_i\in\RR^{5\times 2}$ in the first hidden layer, $A_i\in\RR^{3\times 5}$ in the second hidden layer, and $A_i\in\RR^{1\times 3}$ in the output layer.
	Therefore, we need at most
	\begin{equation*}
		9p\cdot((5\cdot 2+5)+(3\cdot 5+3)+(1\cdot 3 +1))=O(p)
	\end{equation*}
	non-zero parameters.
\end{proof}

	\section{Discussion}\label{sec:Discussion}

\subsection{Comparison to the Result of \citet{Koutschan2023}}
\citet{Koutschan2023} showed the following result:
\begin{theorem}%
	\label{thm:NN_representation_Koutschan}
	Any CPA function $f:\RR^d\to\RR$ with $n$ affine components can be represented by a $\relu$ neural network of depth $\ceil{\log_2(d+1)}+1$ with $O(n^{d+1})$ neurons per layer.
\end{theorem}
Compared to \autoref{thm:NN_representation}, their proof follows a more algebraic approach, leading to dependence on the number of affine components rather than on the number of pieces.
For a CPA function $f:\RR^2\to\RR$ with $p$ pieces and $n$ affine components, \citet{zanotti2025pieces} showed that a constant $c\in\RR$ exists such that  
\begin{equation*}%
	n\leq p \leq c n^3.
\end{equation*}  
Thus, \autoref{thm:NN_representation} implies a neural network representation with width $O(n^3)$ in terms of affine components, which matches the result of \citet{Koutschan2023}.
Moreover, \citet{zanotti2025pieces} also showed that there exist $\beta,c>0$ such that, for every $n\in\NN$, there exists a $\CPA$ function $f:\RR^2\to\RR$ with at most $n$ affine components and at least $\beta \cdot n^{3-\frac{c}{\sqrt{\log_2(n)}}}$ pieces.
For such functions, the construction presented in this work does not provide a significant improvement over \autoref{thm:NN_representation_Koutschan}.
However, there also exist CPA functions with $p=n$, for which \autoref{thm:NN_representation} gives a neural network representation with width $O(n)$, improving upon the $O(n^3)$ bound of \autoref{thm:NN_representation_Koutschan}.

\subsection{Comparison with Convex Pieces}
Bounded polygons without holes can be triangulated without introducing additional vertices, ensuring that each edge is part of the triangulation \citep{ORourke1987ArtGallery}.
For unbounded polygons, each end can be closed by adding four vertices, as illustrated in \autoref{fig:closingEnd}.
The bounded part of the polygon can then be triangulated, while the unbounded part of each end can be subdivided into three convex polygons. 
This procedure produces a refinement of any set of pieces into convex pieces.

For $f\in\CPA_p$, let $p$, $e$, and $v$ denote the numbers of pieces, edges, and vertices corresponding to the admissible set of pieces given by \autoref{lem:sparse_P_V_and_E}.
Applying the aforementioned procedure results in a subdivision with $p_t$ convex pieces (primarily triangles), $e_t$ edges, and $v_t$ vertices.
Since each end is bounded by at least one edge, and any edge bounds at most four ends, the pieces of $f$ collectively have at most $4e$ ends.
Consequently, the convex subdivision satisfies $v_t\leq v+4\cdot 4e$.
By \autoref{lem:sparse_P_V_and_E}, $v\leq 2p$ and $e\leq 3p$.
Using Euler's formula, we find $p_t\leq 2v_t$, implying that there exists $c\in\NN$ such that $p_t\leq c p$, where $c$ is independent of $f$.
Therefore, any $f\in\CPA_p$ can be represented with an admissible set of at most $c\cdot p$ convex pieces.

Starting with such a refinement into convex pieces would also yield results with linear complexity in $p$. 
While this approach simplifies some proofs, it results in worse constants. 
But more importantly, it sacrifices insights into the natural notion of pieces introduced here, which has not been previously considered in the literature. 
For instance, the precise form of the conic decomposition of non-convex polygons (\autoref{lem:InsideOutside}) may be of independent interest.

\subsection{Extension to Higher Dimensions}
In the case of $\RR^2$ studied in this paper, the pieces of a CPA function form a subdivision of the plane consisting of vertices, edges, and regions.
More generally, we can think of these as 0-, 1-, and 2-dimensional \emph{faces}, respectively. In higher dimensions $d>2$, the pieces of a CPA function $f$ similarly induce a subdivision of the input space, consisting of faces of all dimensions $0,...,d$.

Loosely speaking, analogously to the functions $f^v$ and $f^e$ in \autoref{lem:CPA_Decomp}, one can introduce a face-function for each face $F$ of the subdivision as a CPA function $f^F$ whose affine components agree on $F$ and which satisfies $f^F=f$ locally around $F$. 
A higher-dimensional analogue of the presented approach then involves decomposing $f$ into a linear combination of all face-functions of dimension less than $d$.

\citet[Proposition 18 together with Proposition 6]{Tran2024MinimalTRF} proved that, up to an affine function, such a decomposition exists if $f$ is convex, using the formalism of tropical hypersurfaces. 
By refining the subdivision to consist of convex faces, this result can be extended to the non-convex case.

To construct a shallow neural network representation, the next step is to represent each face-function as a shallow network and then stack them into a single network implementing the linear combination.
For a fixed dimension $d$, there exist subdivisions with $p$ regions such that the number of vertices grows as $\Theta(p^{\floor{\frac{d+1}{2}}})$. 
These subdivisions arise as Schlegel diagrams of the polars of cyclic ($d+1$)-polytopes. 
Schlegel diagrams are regular, meaning that there exist convex CPA functions whose pieces realise these diagrams, implying the existence of CPA functions with $p$ pieces and $\Theta(p^{\floor{\frac{d+1}{2}}})$ vertices.

It seems plausible that the total width required for the subnetworks representing vertex-functions cannot asymptotically be smaller than the number of vertices.
As a result, the best possible upper bound on network width that a higher-dimensional analogue of the presented approach could achieve is $O(p^{\floor{\frac{d+1}{2}}})$. 
If $p$ is equal to the number of affine components $n$, this would represent an improvement over the $O(n^{d+1})$ bound on network width established by \citet{Koutschan2023}.
However, \citet{zanotti2025pieces} presents a family of CPA functions for which $p$ behaves roughly like $n^{d+1}$. In these cases, the presented approach does not lead to any improvement.

Nevertheless, the presented ideas may be useful for representing specific subclasses of higher-dimensional CPA functions where the underlying subdivisions have manageable combinatorics.
In particular, a geometric approach can be advantageous when the number of faces is $O(p)$ and the number of pieces incident to each face is bounded by a constant.
In this case, each face-function requires only constant width, allowing the linear combination expressing the original function to be represented by a neural network of width $O(p)$.
Similar structural assumptions have already been employed by \citet{He2020FEM}, who consider shape regular finite element triangulations.

	\section*{Acknowledgments}
I would like to thank Henning Bruhn-Fujimoto for helpful discussions and for introducing me to the topic of this paper.

	\printbibliography
\end{document}